\documentclass[twoside]{article}

\usepackage{aistats2022}

\usepackage[round]{natbib}
\renewcommand{\bibname}{References}
\renewcommand{\bibsection}{\subsubsection*{\bibname}}
\usepackage{threeparttable}
\usepackage{multirow}
\usepackage{hyperref}
\usepackage{url}
\usepackage{booktabs}
\usepackage{algorithm,algorithmic}
\usepackage{amsmath,amssymb,amsthm,amscd,amstext}
\usepackage{wrapfig}
\usepackage{capt-of}

\newcommand{\bm}{ \boldsymbol }

\newcommand{\Bmat}{{\bf B}}

\newcommand{\Dmat}{{\bf D}}

\newcommand{\Lmat}{{\bf L}}

\newcommand{\Wmat}{{\bf W}}

\newcommand{\gv}{{\boldsymbol g}}

\newcommand{\Wv}{{\boldsymbol W}}
\newcommand{\xv}{{\boldsymbol x}}

\newcommand{\zv}{{\boldsymbol z}}

\newcommand{\deltav}{{\boldsymbol \delta}}

\newcommand{\phiv}{{\boldsymbol \phi}}

\usepackage{soul}
\newcommand{\iid}{{\it i.i.d. }}

\newcommand{\CC}{\mathcal{C}}

\newcommand{\var}{\text{Var}}

\newcommand{\BR}{\mathbb{R}}

\DeclareMathOperator{\argmin}{\arg\min}

\usepackage{tikz}
\theoremstyle{plain}
\newtheorem{thm}{Theorem}[section] % reset theorem numbering for each chapter
\theoremstyle{definition}
\newtheorem{defn}[thm]{Definition} % definition numbers are dependent on theorem numbers
\newtheorem{assumption}[thm]{Assumption}
\newtheorem{lem}[thm]{Lemma}

\usepackage[toc,page,header]{appendix}
\usepackage{minitoc}

\doparttoc % Tell to minitoc to generate a toc for the parts
\faketableofcontents % Run a fake tableofcontents command for the partocs

\begin{document}

\twocolumn[

\aistatstitle{Finite-Time Consensus Learning for Decentralized Optimization with Nonlinear Gossiping}

\aistatsauthor{Junya Chen${}^{*}$ \And Sijia Wang \And  Lawrence Carin \And Chenyang Tao${}^{*}$}

\aistatsaddress{ Duke University \And  Duke University \And KAUST \And Duke University} ]

\begin{abstract}
Distributed learning has become an integral tool for scaling up machine learning and addressing the growing need for data privacy. 
Although more robust to the network topology, decentralized learning schemes have not gained the same level of popularity as their centralized counterparts for being less competitive performance-wise. In this work, we attribute this issue to the lack of synchronization among decentralized learning workers, showing both empirically and theoretically that the convergence rate is tied to the synchronization level among the workers. Such motivated, we present a novel decentralized learning framework based on nonlinear gossiping (NGO), that enjoys an appealing finite-time consensus property to achieve better synchronization. We provide a careful analysis on its convergence and discuss its merits for modern distributed optimization applications, such as deep neural networks. Our analysis on how communication delay and randomized chats affect learning further enables the derivation of practical variants that accommodate asynchronous and randomized communications. To validate the effectiveness of our proposal, we benchmark NGO against competing solutions through an extensive set of tests, with encouraging results reported.
\end{abstract}

\vspace{-1em}
\section{Introduction}
\vspace{-5pt}
The growth of data and the complexity of modern machine learning problems have often necessitated the deployment of distributed learning algorithms to scale up implementation in practice \citep{bottou2010large, dean2012large}. Such learning schemes allow the workload to be split across multiple computing units, commonly known as {\it workers}, to parallelize the computation. This enables learning within a fraction of the time typically required for a single unit, and also adds robustness as failures of individual workers can be easily isolated. Apart from its efficiency, applications of distributed learning are also frequently driven by data-privacy \citep{balcan2012distributed,truex2019hybrid,kairouz2019advances} or cost-performance \citep{kovalev2021linearly,haddadpour2021federated} considerations.

Crucial for all distributed-learning schemes is to establish a sense of consensus among individual workers. 
Two major strategies are employed to reinforce the desired coherence: centralized and decentralized synchronizations. As its name suggests, the former makes use of centralized nodes, commonly known as {\it parameter servers} \citep{li2014communication}, to aggregate model updates from and broadcast the latest parameters back to the workers, thereby enforcing model parallelism ({\it i.e.}, exact synchronization). Centralized schemes are considered less scalable \& more vulnerable due to communication bottlenecks \& heavy reliance on central nodes, and more demanding for the network infrastructure.

In the decentralized setup, workers communicate only with their neighbors rather than the centralized server \citep{nedic2018network}. Parameter updates are executed on individual workers using local gradients, followed by a consensus update step to integrate information from their neighbors. 
Done properly, all workers will agree on a set of parameters that is close to the optimum \citep{Wu2017}. In other words, extra flexibility with the network topology can be achieved by settling with asymptotic or approximate synchronization.

For practical considerations, however, the efficiency afforded by distributed learning may be overshadowed by the cost of communication, thus preventing ideal linear scaling predicted by theory \citep{dean2012large}. The popularity of large-scale neural networks further aggravates this issue, in which the synchronization latency for big models may easily wash away the potential gain from distributed computation \citep{shamir2014fundamental,leblond2017asaga}. To ameliorate this difficulty, research has focused on two fundamental questions: how and what to communicate? For the former, the goal is to achieve higher throughput. Strategies like eager updates \citep{tandon2017gradient,lee2017speeding} and asynchronous protocols \citep{recht2011hogwild} are among the most straightforward and effective, though proper compensation may be needed to correct for the altered learning dynamics \citep{mitliagkas2016asynchrony, zheng2017asynchronous}.

As for the latter, it is always desirable to communicate frugally. One of the most successful strategies for synchronization under budget is gradient compression, which entails transmitting a succinct summary instead of the full gradients to reduce communications \citep{lin2018deep}. A major motivation is that model gradients often manifest more redundancy ({\it e.g.}, sparsity) compared with the full set of model parameters \citep{na2017chip,ye2018communication,basu2020qsparse}. Prominent examples from this category include gradient sparsification \citep{wangni2018gradient} and quantized gradients \citep{wen2017terngrad,alistarh2017qsgd, wu2018error,reisizadeh2020fedpaq}.

This study focuses on decentralized learners, and addresses an issue orthogonal to the endeavors discussed above. Specifically, we seek answers on how to synchronize. Our investigation is motivated by the observation that, despite matching theoretical convergence rates, decentralized algorithms often fail to compete with centralized alternatives in practical settings. We hypothesize that this is due to the lack of synchrony among the collaborating workers in the decentralized setting. This is especially the case for over-parameterized models, where small discrepancies in the parameter space sometimes can lead to qualitatively different behavior. Consequently, we wish to keep the local models better synchronized with the same or fewer communications. One promising direction is to alter the communication protocol, a path we take in this work.

This paper revisits decentralized stochastic optimization from a consensus learning perspective, in the hope of better understanding the trade-offs involved and motivating new algorithms for improved practice. Our work yields the following contributions: $(i)$ we propose a new family of consensus learning schemes based on nonlinear gossiping for decentralized optimization; $(ii)$ we analyze the convergence of our solution, and discuss the computation-communication trade-offs under practical considerations; and $(iii)$ we discuss extensions to more general learning settings, such as with randomized and time-varying network topology, and the communication delay case. Our theory generalizes the existing results for decentralized learning, and the proposed algorithms promise to better balance the tension between the communication bottleneck and learner synchrony.

We use the following notations in this paper: In the decentralized setup, the workers participate in peer-to-peer communications defined by the network topology, where each worker is only allowed to exchange information with the connecting workers. Formally, let $\mathcal{G} = \{\mathcal{V}, \Wmat\}$ denote a weighted undirected graph with the set of nodes (or vertices) $\mathcal{V}$. $\Wmat$ is a symmetric matrix called the weight matrix. We say that a set ${v_i, v_j}$ is an edge if $w_{ij} > 0$. 
For every node $v_i\in \mathcal{V}$, the degree $d(v_i)$ is the sum of the weights of the edges adjacent to $v_i$: $d(v_i)=\sum_{j=1}^n \Wmat_{ij}$. The degree matrix $D(\mathcal{G})$ is the diagonal matrix $\Dmat = \text{diag}\left(d_1, \cdots, d_n\right)$. The graph Laplacian is defined by $\Lmat \triangleq \Dmat - \Wmat$, whose eigenvalues encode information on how network topology affects convergence. $\lambda_i(\cdot)$ denotes the $i$-th smallest eigenvalue of a matrix.

\begin{figure*}[t]
	\begin{center}
	    \includegraphics[width=0.22\textwidth]{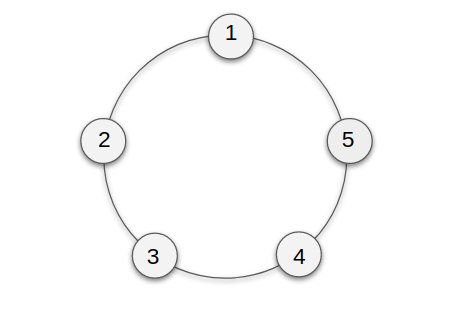}
		\includegraphics[width=0.74\textwidth]{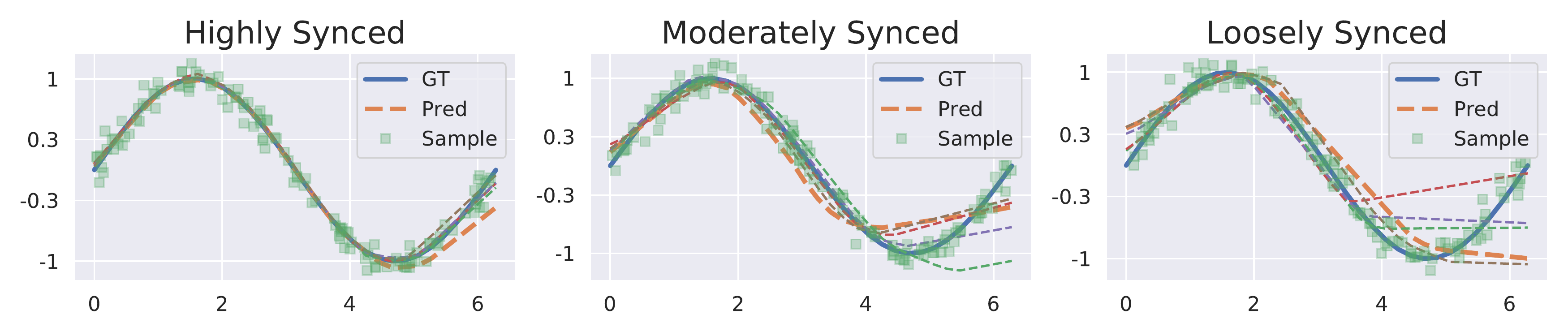}
        \vspace{-1em}
		\caption{Linear Gossiping SGD with varying synchronization level under the \iid setup on toy model. (a) computation graph with a ring topology. (b) highly synced. (c) moderately synced. (d) loosely synced. With the same number of training iterations, highly synced workers converge better. }
		\label{fig:linear_th}
	\end{center}
	\vspace{-2em}
\end{figure*}

\vspace{-8pt}
\section{Problem Setup}\label{problem_setup}
\vspace{-6pt}
Consider the following distributed optimization problem:$~
\min_{\xv\in \BR^d} \left[\frac{1}{n}\sum_{i=1}^n f_i({\xv})\right]$, where $\xv$ denotes the model parameters to be optimized and $f_i:\BR^d\rightarrow \BR$ for $i \in [n] \triangleq \{1, \cdots, n\}$ are the local objectives defined on each worker. For most modern machine learning applications, $f_i(\xv)$ is given by the empirical expectation
$\label{f_i}
f_i(\xv) = \mathbb{E}_{\xi_i\sim \mathcal{D}_i}[F_i(\xv,\xi_i)]$,
where $\mathcal{D}_i$ denotes the local data  distribution, and $F_i(\xv, \cdot)$ is the local loss function of model parameters $\xv$ on the worker $i$. 

% It is convenient to treat $F_i(\xv,\cdot)$ as a random function, with $f_i(\xv)$ as its expectation. 
% \jy{We work under the assumption that computational resources on each worker prohibit the direct evaluation of $\nabla f_i(\xv)$, necessitating mini-batch learning schemes, such as {\it stochastic gradient descent} (SGD), based on stochastic gradients $\nabla F_i(\xv;\xi_i), \xi_i \sim \mathcal{D}_i$, where $\nabla$ is taken wrt $\xv$. }
We call the setting \iid if all $\mathcal{D}_i$ are independent copies from the same underlying data distribution, and non-\iid otherwise. The non-\iid setting commonly arises when locally curated data differ distributionally ({\it e.g.}, data collected from different countries), and especially the case when data privacy is of concern. 

% The non-\iid setting commonly arises in the presence of domain shift, {\it i.e.}, data collected from different physical locations differ distributionally, and especially the case when data privacy is of concern. 

Before addressing the detailed analysis, we summarize the standard technical conditions and define the key synchronization index function throughout the paper.
\begin{assumption}\label{assumption}
Local objectives $f_i:\BR^d\rightarrow \BR$ for all $i\in [n]$ are $L$-smooth and $\mu$-strongly convex, with the variance of their stochastic gradients uniformly bounded for all $\xv$ and $i$: 
\begin{align}
&\mathbb{E}_{\xi_i\sim \mathcal{D}_i}\left\|\nabla F_i(\xv, \xi_i) - \nabla f_i(\xv)\right\|^2\leq \sigma_i^2,\\
&\mathbb{E}_{\xi_i \sim \mathcal{D}_i}\left\|\nabla F_i(\xv, \xi_i)\right\|^2\leq G^2, \forall \xv\in \BR^d, i\in [n].
\end{align}
\end{assumption}

\begin{assumption}
We assume that $\Wmat \in  [0, 1]^{n\times n}$, $\Wmat$ is a symmetric $(\Wmat = \Wmat^\top)$ doubly stochastic $(\Wmat \mathbf{1} = \mathbf{1}, \mathbf{1}^\top \Wmat = \mathbf{1}^\top$) matrix.
\end{assumption}

\begin{defn}[Synchronization index \citep{olfati2007consensus}]
    Denote $\bm{\delta}_i = \xv_i - \frac{1}{n}\sum_{i=1}^n \xv_i$ as the disagreement vector, and we define the Lyapunov function of (\ref{ct}) as $V = \sum_i\|\bm{\delta}_i\|^2$, which is also called the synchronization index. 
\end{defn}
For a small value of $V$, the parameters of different workers are similar to each other. When all workers are perfectly synchronized, $V$ function reaches zero. The dynamic of $V$ function depicts the convergence behavior of gossiping learning, and serves as an important role in decentralized learning.
% \jy{More theoretical results on V can be found in SM Sec. A.}

\vspace{-8pt}
\subsection{Gossiping SGD}
\vspace{-6pt}
Gossiping averaging is one of the most popular techniques for decentralized learning \citep{koloskova2019decentralized, lian2017asynchronous, nedic2009distributed, tang2018communication}. In this study, we restrict our discussion to the variants based on {\it stochastic gradient descent} (SGD).

In gossiping SGD, there are two major steps involved in iteration $t$: $(i)x^{(t)}\rightarrow x^{(t+\frac{1}{2})}$: {\it SGD update} (local model update) and $(ii)$ $x^{(t+\frac{1}{2})}\rightarrow x^{(t+1)}$:{\it gossiping update} (neighbor communication). The SGD step proceeds as regular SGD on a worker, with local models updated independently based on the gradient computed from a mini-batch of data curated on the individual workers. Under proper regularity conditions, the SGD updates push local models toward the local solution $\xv_i^* \triangleq \argmin_{\xv} \left[f_i(\xv) \right]$.

The gains in scaling and the consensus on the correct  global solution\footnote{Note that in the non-\iid setting, $\{\xv_i^*\}$ differ from each other.}, come from the gossiping updates. Now each worker exchanges messages with their neighbors and updates their models through localized averaging
\vspace{-2pt}
\begin{equation}
	\xv_i^{(t+1)} = \xv_i^{\left(t+\frac{1}{2}\right)} +\gamma\sum_{j:(i,j)\in \mathcal{E}} \Wmat_{ij} \Delta_{ij}^{{\left(t+\frac{1}{2}\right)}}, \forall i\in [n].
\vspace{-2pt}
\end{equation}
where $\Delta_{ij}^{\left(t+\frac{1}{2}\right)} \triangleq \xv_j^{\left(t+\frac{1}{2}\right)} - \xv_i^{\left(t+\frac{1}{2}\right)}$ is the update vector between worker $j$ and worker $i$ at iteration $t$, and $\gamma>0$ controls the synchronization rate for local averaging.

Intuitively, similar to the outright averaging employed by their centralized counterparts, the gossiping update mollifies the noise term introduced from the SGD updates, thereby achieving linear speedup in terms of convergence. Without the SGD updates, all $\xv_i^{(t)}$ converge exponentially to the average $\bar{\xv} = \frac{1}{n}\sum_{i=1}^n \xv_i^{(0)}$ under gossiping updates. The convergence rate depends on the network topology: for a connected undirected network it is faster or equal to $1-\gamma \lambda_2(\Lmat)$ \citep{olfati2007consensus,chen2017second}. $\lambda_2(\Lmat)$ is known as the algebraic connectivity of $\mathcal{\Wmat}$.

\vspace{-8pt}
\subsection{Synchronization and convergence}
\vspace{-6pt}
Now let us compare centralized optimized and decentralized linear gossiping. Our goal is to underscore that better synchronization provides a faster overall convergence rate.

To motivate, consider the following hypothetical example: two populations of workers, with the same parameter mean $\bar{\xv}$ but different synchronization, {\it i.e.}, $\var[\xv_{i}]$ differs. Typically $\bar{\xv}$ is identified as the current global solution that one seeks to optimize. For each local gradient $\partial \xv_{i}$, we could treat it as a noisy approximation $\hat{\partial} \xv$ of the target gradient $\partial \xv$ evaluated at $\bar{\xv}$. Let us execute one step of SGD and infinite steps of gossiping, then it essentially reduces to applying gradient descent with $\hat{\partial} \xv$ instead of $\partial \bar{\xv}$. As $\var[\xv_{i}]$ becomes smaller ($\{\xv_i\}$ more concentrated around $\bar{\xv}$), the approximation to $\partial \bar{\xv}$ gets better (Figure \ref{fig:grad_compare}). Consequently, one should expect better convergence provided the workers maintain a more synchronized status.

% To formalize this intuition, we introduce the synchronization parameter $\nu_t$ along the optimization path $\{ \{\xv_i^t\}_i \}_t$, and commensurately define the minimum synchronization parameter through time horizon $T$, as follows:
% \begin{equation}
% \begin{array}{l}
% \nu_t^2= \min\left\{\frac{\left\|\sum_{i=1}^n \nabla f_i(\bar{\xv}^t) \right\|^2}{\left\|\sum_{i=1}^n \nabla f_i(\xv_i^t) - \sum_{i=1}^{n}\nabla f_i(\bar{\xv}^t)\right\|^2} , \frac{1}{\eta_t^{2\varepsilon}} \right\},\\
%     [15pt]
%     \nu_T^{\min}= \displaystyle{\min_{t \in [T]}} \{ \nu_t \}, \mbox{where}~0<\varepsilon<1.
% \end{array}
% \end{equation}

We verify this intuition experimentally in Figure \ref{fig:linear_th}. We use ring topology (left above in Figure \ref{fig:linear_th}) with $n=5$ workers to solve a simple least square regression problem with \iid data ({\it i.e.}, each worker only gets data samples just from a non-overlapping partition of data distribution). 
% We distribute the $m$ data samples evenly among the $n$ workers and consider the non \iid setting, where each worker only gets data samples just from a non-overlapping partition of data distribution.
We set different thresholds ($V_{th}$ = 1, 5, 10) for the Lyapunov function to measure synchronization. After each SGD update, we execute gossiping updates until the tolerance is satisfied ($V\leq V_{th}$). We see that the synchronization level among workers strongly correlates with the distributed learning performance.

The following statement formalizes the intuition above, that the level of synchronization among the workers affects the convergence rate. 
% \cytao{Maybe an separate paragraph for Lyapunov function. and some refs.}

\begin{thm}[Convergence and synchronization]\label{centralized}
	Under Assumptions \ref{assumption}, centralized SGD and linear gossiping SGD with step-size $\eta_t = \frac{4}{\mu(a+t)}$, for parameter $a\geq 16\frac{L}{\mu}$, converges at the rate
	\vspace{-2pt}
	\begin{equation}\label{eq:sync_conv}
		f\left(\xv_{avg}^T\right) - f^*
    		\leq \underbrace{\frac{\mu a^3}{8S_T}}_{\text{I}} +\underbrace{\frac{4\bar{\sigma}^2T\left(T+2a\right)}{\mu n S_T}}_{\text{II}}+\underbrace{h(T)}_{\text{III}}
    \vspace{-2pt}
	\end{equation}
	where $\bar{\sigma}^2 = \frac{1}{n}\sum_i \sigma_i^2$, $\xv_{avg}^T = \frac{1}{S_T}\sum_{t=0}^{T-1}w_t\xv^t$ for weights $w_t = (a+t)^2$, and $S_T = \sum_{t=0}^T w_t\geq \frac{1}{3}T^3$. The synchronization-dependent terms are given by
	\begin{equation*}
	h(T) = \begin{cases}
	0,~~ \hspace{6.em} \mbox{Centralized SGD}\\
	\frac{(2L+\mu)}{nS_T}\sum_{t=0}^T w_tV^{(t)},~~\mbox{Gossiping SGD}
	\end{cases}
	\end{equation*}
	where $V^{(t)} = \sum_i\|\xv_i^{(t)}-\bar{\xv}\|^2$.
\end{thm}

Now we are ready to discuss the interplay between synchronization and algorithmic convergence. After a sufficient number of iterations ({\it i.e.}, a fairly large $T$), one recognizes the second term from the bracket in (\ref{eq:sync_conv}) is the leading term. For highly synchronized workers (small $V^{(t)}$), the gossip SGD scheme matches its centralized counterpart. As such, provided a sufficient communication budget, one can maximally enforce the synchronization to expedite convergence.

\begin{figure}[t]
	\begin{center}
		\includegraphics[width=.2\textwidth]{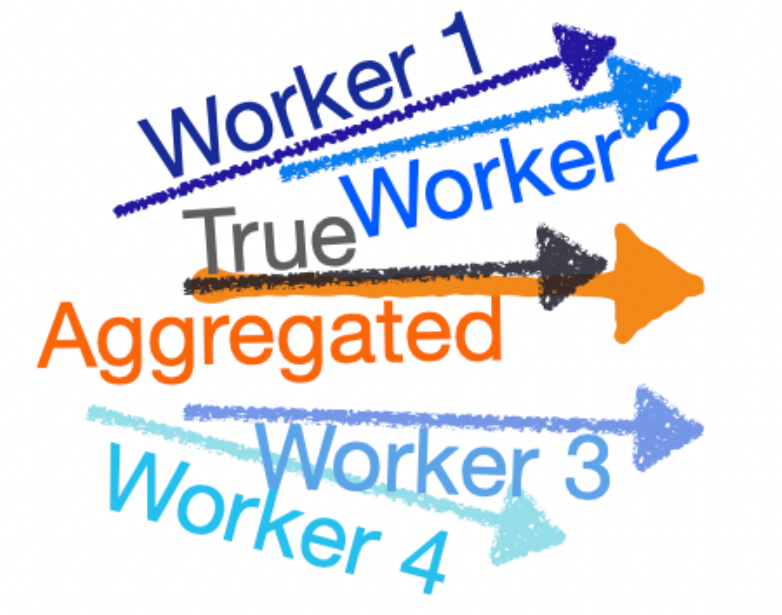}\includegraphics[width=.25\textwidth]{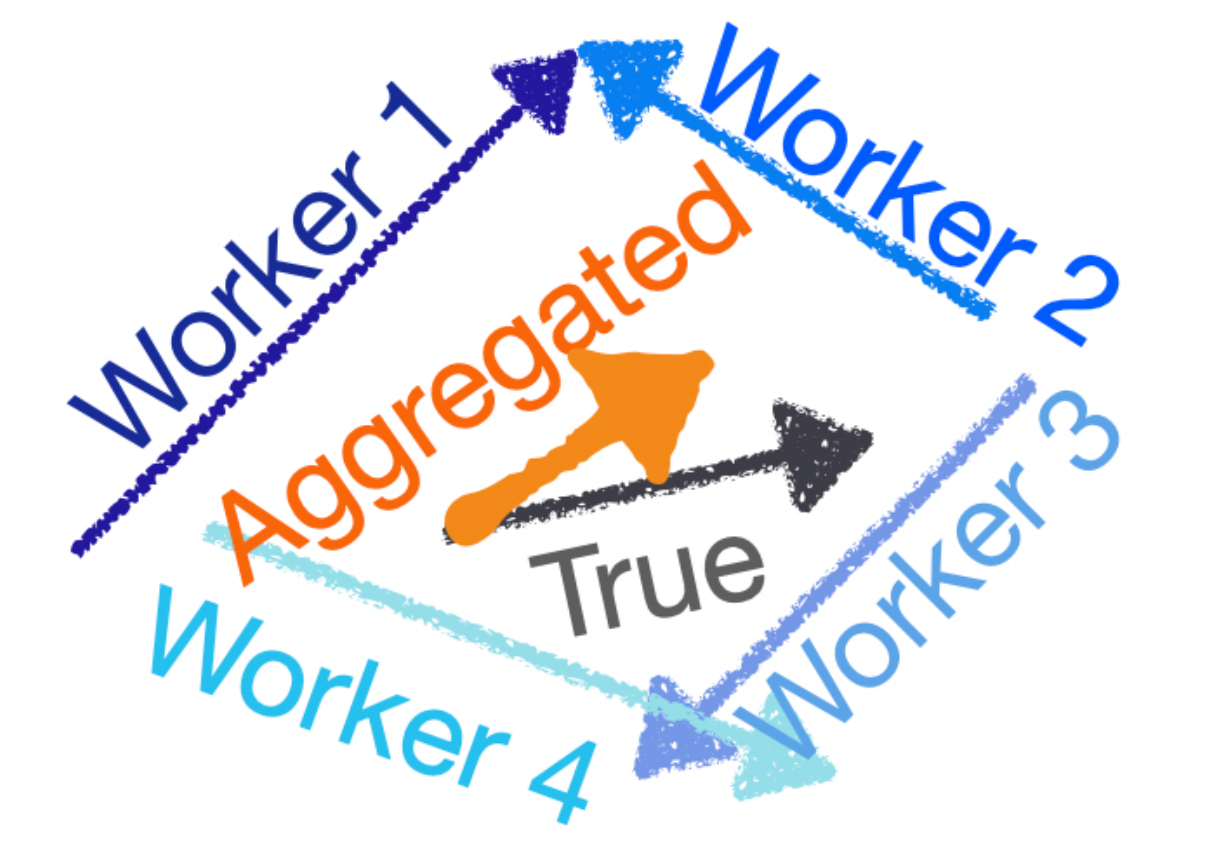}
        \vspace{-1em}
		\caption{When workers are synced (close to each other in the state space, left plot), their gradients will be more aligned with that of the global average; if they are less coordinated (right), individual gradients will be more variable.
% 		(left) When the workers are concentrated in a small region (the start point of the arrows close to each other), the gradients on each worker are more likely to point at similar direction, which leads to a better aggregation gradient. (right) The gradients on workers cancel out in a large region, delivers an inaccurate aggregated gradients with large variance.
		}
		\label{fig:grad_compare}
	\end{center}
	\vspace{-1.5em}
\end{figure}

\vspace{-8pt}
\section{Decentralized Learning with Nonlinear Gossiping}\label{sec:decentralized}
\vspace{-6pt}
Above we have established the intuition that decentralized optimization is expected to benefit from improved consensus among workers. Inspired by this, we present NGO, a novel distributed optimization scheme based on {\it {\bf N}onlinear {\bf Go}ssiping} (NGO) consensus protocol. While most of the existing literature focuses on how to improve the communication efficiency for decentralized learning, we investigate a novel direction: how to improve synchronization by altering the communication updates. This implies our NGO can be combined with prior arts to further enhance learning efficiency. Below we elaborate our construction, relegating all technical proofs to the Supplementary Material (SM) Sec. A.

\vspace{-8pt}
\subsection{Finite-time consensus protocol}
\label{subsec:finite-time}
\vspace{-6pt}

To motivate the new consensus updates, we first briefly review basic consensus schemes in the continuous time limit. Consider a network with $n$ nodes, continuously coupled in the following manner:
\vspace{-5pt}
\begin{equation}\label{ct}
\dot{\xv}_i(t) = \sum_{j:(i,j)\in \mathcal{E}} \Wmat_{ij}\left(\xv_j(t)-\xv_i(t)\right)
\vspace{-.1em}
\end{equation}
Note the discretized version of (\ref{ct}) corresponds to the consensus updates employed by linear gossiping\footnote{We use $f(t)$ to denote continous dynamics.}.
One can readily show that following the dynamics described in (\ref{ct}), $V(t)$ converges to zero in the asymptotic time limit if the graph is connected ($t \rightarrow \infty$; see SM Sec. A), and the workers reaches consensus\footnote{We use consensus and synchronization interchangeably.}. For consensus learning, we desire that perfect synchronization can happen sooner, preferably within a finite time frame. The following statement verifies that this is possible through use of nonlinear couplings. 

% \begin{figure}[h]
% 	\begin{center}
	
% 		\vspace{-1.5em}
% 		\caption{Convergence analysis of nonlinear ($p\neq 1$) and linear ($p=1$) couplings with ten workers. (a) Sub-linear coupling ($p=0.5$). (b) Linear coupling ($p=1$). (c) Super-linear coupling ($p=2$). (d) Lyapunov function. \cytao{Maybe highlight superlinear coupling converges faster when there is a larger discrepancy btw the workers, but the opposite is true when the workers are more synced.}}
% 		\cytao{I guess what I wanted to show is that superlinear drops fast first ($\|x\|>1$), then it struggle to make process for $\| x \|<1$.}
% 		\jy{Still affect the margin in this figure. Superlinear case needs further clarification.}
% 		\vspace{-1.5em}
% 		\label{fig:compare}
% 	\end{center}
% \end{figure}

\begin{figure*}[h]
	\begin{center}
		\includegraphics[width=\textwidth]{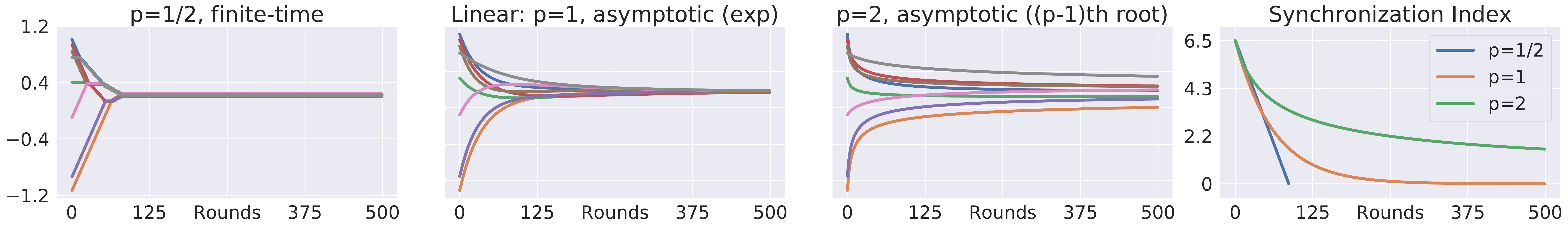}
		\vspace{-2em}
		\caption{Convergence rate of nonlinear ($p\neq 1$) and linear ($p=1$) couplings with ten workers. (a) Sub-linear coupling ($p=0.5$). (b) Linear coupling ($p=1$). (c) Super-linear coupling ($p=2$). (d) Synchronization Index. Finite time consensus is achieved with sub-linear coupling.}
% 		\cytao{I guess what I wanted to show is that superlinear drops fast first ($\|x\|>1$), then it struggle to make process for $\| x \|<1$.}\jy{put that in supp}
		\vspace{-1.5em}
		\label{fig:compare}
	\end{center}
\end{figure*}

\begin{thm}[Finite time consensus \citep{lu2016note}]
% [Finite time consensus via nonlinear protocol \citep{lu2016note}]\label{finite_consensus}
Consider the following system with nonlinear coupling $\phiv(\xv): \BR^d \rightarrow \BR^d$
\vspace{-2pt}
\begin{equation}
\dot{\xv}_i(t) = \sum_{j:(i,j)\in \mathcal{E}} W_{ij}\phiv(\xv_j(t) - \xv_i(t)),
\vspace{-2pt}
\end{equation}
where $\phiv(\zv) = [\text{sign}(\zv_1)|\zv_1|^{2p-1}, \cdots, \text{sign}(\zv_n)|\zv_n|^{2p-1}]$ for $p \in [\frac{1}{2},1)$. 
Then $V(t) = 0$ for all 
\vspace{-2pt}
\begin{equation}\label{eq:finite_time}
    t>
    T^*(\Wmat, p) \triangleq \frac{\left(\sum_{i}\left(\xv_i(0)-\bar{\xv}(0)\right)^2\right)^{1-p}}{4\gamma (1-p) \left(\lambda_2(\Lmat(\Bmat))\right)^p }
\vspace{-2pt}
\end{equation} 
$T^*(\Wmat, p)$ is a function of the nonlinear parameter $p$ and the network topology $\Wmat$, with $\Bmat =\left[\Wmat_{ij}^{\frac{1}{p}}\right]$,  $\lambda_2(\Lmat(\Bmat))$ is the algebraic connectivity of $\mathcal{G}(\Bmat)$.
\end{thm}
To illustrate, we randomly sample ten samples from Gaussian distribution to denote the inital state of ten workers. After evolving for 500 rounds according to Eqn (\ref{eq:finite_time}), we plot the convergence results with $p=0.5$ (sublinear), $p=1$ (linear), $p=2$ (superlinear), along with the corresponding synchronization index In Figure \ref{fig:compare}. We see that the sublinear coupling reaches consensus in finite time, follows by the linear coupling (exponentially) and lastly the superlinear coupling. 
\vspace{-8pt}
\subsection{Dencentralized nonlinear gossiping}\label{subsec:ngo}
\vspace{-6pt}

Inspired from above, we consider the following nonlinear gossiping protocol after the SGD update
\vspace{-5pt}
\begin{equation}\label{x_i}
    \xv_i^{(t+1)} = \xv_i^{(t+\frac{1}{2})} + \gamma \sum_{j:(i,j)\in \mathcal{E}} \Wmat_{ij} \phiv\left(\xv_j^{(t)} - \xv_i^{(t)}\right).
    \vspace{-5pt}
\end{equation}
We summarize the pseudocode of a simple implementation of decentralized NGO in Algorithm \ref{alg_nonlinear}, and establish its convergence rate as follows. 

\begin{algorithm}[H]
\caption{Nonlinear Gossiping (NGO) Learning} 
	\label{alg_nonlinear}
\begin{algorithmic}
	\REQUIRE {Initialize value $\bar{\xv}^{(0)}\in \mathbb{R}^d$, $\xv_i^{(0)}=\bar{\xv}^{(0)}$ on each node $i\in [n]$, consensus step-size $\gamma$, learning rate $\eta_t$, communication graph $\mathcal{G} = (\mathcal{V}, \Wmat)$, and number of total budget $T$.}
	\FOR {$t = 0,\cdots, T-1$} 
	\STATE{\{In parallel for all workers $i\in [n]$\}}
	\STATE{\small{Sample $\xi_i^{(t)}$, compute gradient $\gv_i^{(t)} = \nabla F_i\left(\xv_i^{(t)}, \xi_i^{(t)}\right)$}.}
	\STATE{$\xv_i^{\left(t+\frac{1}{2}\right)} = \xv_i^{(t)} -\eta_t\gv_i^{(t)}$.}			
	\FOR {neighbors $j:\{i,j\}\in \mathcal{E}$}
	\STATE{Send $\xv_i^{\left(t+\frac{1}{2}\right)}$ and receive $\xv_j^{\left(t+\frac{1}{2}\right)}$.}
	\ENDFOR	
	\STATE{$\xv_i^{(t+1)} = \xv_i^{\left(t+\frac{1}{2}\right)} + \gamma \sum_{j:(i,j)\in \mathcal{E}}\Wmat_{ij}\phiv \left(\Delta_{ij}^{\left(t+\frac{1}{2}\right)}\right)$, where $\Delta_{ij}^{\left(t+\frac{1}{2}\right)}=
\xv_j^{\left(t+\frac{1}{2}\right)} - \xv_i^{\left(t+\frac{1}{2}\right)}.$}
	\ENDFOR
 	\end{algorithmic}
\end{algorithm}

% \begin{figure}[t]
% 	\begin{center}
% 	\hspace{-1.5em}
% 		\includegraphics[width=0.5\textwidth]{icml-2020/Nonlinear_gossip.pdf}
% 		\vspace{-1em}
% 		\caption{Comparison of consensus protocols. (a) Linear Gossip. (b) Nonlinear Gossip.}
% 		\label{fig:linear}
% 	\end{center}
% 	\vspace{-1em}
% \end{figure}

\begin{thm}[Convergence of NGO]\label{main thm}
With hyperparameter $a \geq \max\{\frac{5}{\beta}, 15 \kappa\}, \kappa=\frac{L}{\mu}$, and learning rate scheduling $\eta_t =\frac{4}{\mu(a+t)}$, 
nonlinear gossiping learning converges at the rate 
% \begin{align*}
% f\left(\xv^{(T)}_{\text{avg}}\right)-f^*
% \leq \frac{\mu a^3}{8S_T}\|\bar{\xv}^{(0)}-\xv^*\|^2+\frac{4T(T+2a)}{\mu S_T}\frac{\bar{\sigma}^2}{n} \\
% + \frac{2L+\mu}{S_T}\sum_{t=0}^{T^*(\Wmat, p)}\sum_{i=1}^n \frac{(a+t)^2}{n} \|\xv_i^{(t)}-\bar{\xv}\|^2
% \end{align*}
\vspace{-1em}
\begin{align*}
f\left(\xv^{(T)}_{\text{avg}}\right)-f^*
&\leq \frac{\mu a^3}{8S_T}\|\bar{\xv}^{(0)}-\xv^*\|^2+\frac{4T(T+2a)}{\mu S_T}\frac{\bar{\sigma}^2}{n} \\
&+ \frac{2L+\mu}{nS_T}\sum_{t=0}^{T^*(\Wmat, p)} w_t V^{(t)}
\vspace{-1.5em}
\end{align*}
where $\xv^{(T)}_{\text{avg}} = \frac{1}{S_T}\sum_{t=0}^{T-1}w_t\bar{\xv}^{(t)} 
$ for weights $w_t = (a+t)^2$, and $S_T= \sum_{t=0}^{T-1}w_t\geq \frac{1}{3}T^3$, with $T^*(\Wmat, p)$ is defined in (\ref{eq:finite_time}).
\end{thm}

% We compare the performance of linear gossip scheme and our NGO under the non-\iid setting in Figure \ref{fig:compare}. 
% The proposed nonlinear gossip converges after 20 iterations while linear gossip does not converge at the same iterations. 
\vspace{-8pt}
\subsection{Consensus with asynchronous or randomized communications}\label{subsec:consensus}
\vspace{-6pt}
For practical considerations, it is often desirable that a solution should be robust to network latency and possible worker or network failures. For example, some workers may take longer to finish their local computation (commonly known as {\it stragglers}), which blocks the overall progress if faster workers have to wait for the synchronous updates. This is very common in heterogeneous computing environments, where decentralized learning schemes are mostly employed. On the other hand, the network may be subject to a total communication budget ({\it e.g.}, bandwidth), such that a worker can only afford to communicate with a pre-defined number of its neighbors \citep{koloskova2020unified}.

We investigate $(i)$ asynchronous consensus updates, where stale model parameters are used for consensus updates; and $(ii)$ random communication graphs, where workers can only communicate to sub-sampled neighbors at each iteration. The following statement asserts that it suffices to prove the consensus protocols will achieve synchronization by themselves. 

\begin{lem}
\label{lemma}
	The average $\bar{\xv}^t$ of the Gossiping-type SGD algorithm satisfies the following:
	\vspace{-.2em}
	\begin{equation*}
	\begin{aligned}
	&\mathbb{E}_{\xi_1^{(t)}, \cdots, \xi_n^{(t)}}\left\|\bar{\xv}^{(t+1)} -\xv^*\right\|^2\\
	\leq& \left(1-\frac{\mu\eta_t}{2}\right)\left\|\bar{\xv}^{(t)}-\xv^*\right\|^2 +\frac{\eta_t^2\bar{\sigma}^2}{n} \\
	&-2\eta_t\left(1-2L\eta_t\right)e_t
	+\eta_t\frac{2\eta_tL^2+L\eta_t+\mu}{n}V^{(t)}
	\end{aligned}
	\vspace{-.5em}
	\end{equation*}
where $e_t = \mathbb{E}[f(\xv^{(t)})] - f^*$.
\end{lem}

\begin{lem}
Gossiping-type SGD will converge provided the consensus protocol can reach synchronization.
\end{lem}

{\it Sketch of proof.} Since the consensus protocol can reach synchronization, the Lyapunov function $V^{(t)}$ can be bounded. Inserting this bound into (\ref{lemma}), via an application of contraction mapping, the convergence of $e_t$ can be established, {\it  i.e.}, $f(\xv^{(t)})\rightarrow f^*$ as $t\rightarrow \infty$. \hfill $\square$

{\bf Nonlinear consensus with communication delays.} To avoid technical clutter, we consider the following simplified nonlinear gossiping consensus problem, with delayed communications.  
\vspace{-.2em}
\begin{equation}\label{delay}
\xv_i^{(t+1)}  = \xv_i^{(t)} + \sum_{j:(i,j)\in \mathcal{E}} \Wmat_{ij} \phiv\left(\xv_j^{(t-\tau)} - \xv_i^{(t-\tau)}\right).
\vspace{-.5em}
\end{equation}
The following statement gives the maximum tolerance of communication delays for consensus to work in terms of the spectrum of the graph Laplacian. 
\begin{thm}[NGO convergence with stale updates]
The consensus update defined by (\ref{delay}) asymptotically solves the average consensus problem with a uniform time-delay $\tau$ for all initial states if and only if $0 \leq \tau\leq  \pi/{2\lambda_n(\Lmat)}$.
\end{thm}

\paragraph{Nonlinear consensus on random graphs} We further study how the consensus protocol cope with communication uncertainties. In particular, we consider the following random graph model, as a tractable mathematical simplification of our problem. Let $u$ be the incidence rate of an edge, and the adjacency matrix $\Wmat^{(t)}(u)$ for graph $G \in \mathcal{G}(n,u)$. Define the following model
\vspace{-.8em}
\begin{equation}%[Convergence for randomized NGO]
\label{aijp}
	\Wmat_{ij}^{(t)}(u) = \begin{cases}
	1, ~~\mbox{with prob. $u$}, \\
	0, ~~\mbox{with prob. $1- u$}.\\
	\end{cases}
	\vspace{-.5em}
\end{equation}
For each time $t$, we have $\Wmat^{(t)}(u)$ as an \iid draw from the above model, which poses NGO with randomized communications as the following random dynamic system:
\vspace{-.5em}
\begin{equation}\label{eq:random}
\xv_i^{(t+1)}  = \xv_i^{(t)} + \sum_{j:(i,j)\in \mathcal{E}} \Wmat_{ij}^{(t)}(u) \phiv\left(\xv_j^{(t)} - \xv_i^{(t)}\right).
\vspace{-.5em}
\end{equation}
Note this model subsumes cases with random communication failures or limited communications discussed at the beginning of this subsection. The following statement verifies its consensus property. 
\begin{thm} The random dynamic system defined by (\ref{aijp})-(\ref{eq:random}) reaches consensus with probability $1$.
\end{thm}

\vspace{-8pt}
\section{Related Work}
\label{sec:rel}
\vspace{-6pt}

{\bf Stochastic gradient descent} is the {\it de facto} practice in solving large-scale learning problems, which iteratively updates model parameters with noisy gradients \citep{zhang2004solving, bottou2010large, ghadimi2013stochastic}. Empirically, speedy convergence is expected from the SGD-type algorithm due to a better trade-off between faster iterations and a slightly slower convergence rate \citep{ruder2016overview, li2014efficient}. 
Scaling up SGD with parallel computation is receiving growing attention in recent years \citep{bordes2009sgd, lian2017can, jahani2018efficient, qu2017accelerated, yu2019linear}, and has witnessed great success \citep{goyal2017accurate}. Other progress in distributed SGD setting has been made, {\it e.g.}, addressing the robustness to malicious inputs \citep{alistarh2018byzantine} and how to make performance-adaptive updates \citep{teng2019leader}. The proposed NGO bears resemblance to recent nonlinear variants of SGD, such as sign-SGD \citep{bernstein2018signsgd}. The difference lies in we apply nonlinearity to the consensus protocol, rather than model updates, yet synergies can be exploited.

{\bf Decentralized coordination} is a naturally occurring phenomenon observed in many physical systems, such as flocking of birds and schooling of fish. Consequently, the study of network coordination with dynamic agents is an active topic in dynamical systems \citep{kempe2003gossip, xiao2004fast}.
The seminal work of \citet{olfati2004consensus} first established convergence analysis of non-trivial consensus protocols for a network of integrators with a directed information flow and fixed or switching topology. A comprehensive review on the topic can be found in \citet{motter2013spontaneous}. More recently, \citet{wang2010finite, lu2016note} investigated finite-time consensus problems for multi-agent systems and presented a framework for constructing effective distributed protocols. In this study, we explore how these more general consensus protocols can be leveraged to benefit distributed learning.

{\bf Decentralized optimization} has attracted lots of attention due to the recent explosion of data and model complexity. Gossiping-type schemes combined with SGD, showing sub-linear convergence to optimality can be achieved via gradually diminishing step-size \citep{nedic2009distributed}. A number of later contributions \citep{shi2015extra,johansson2008subgradient,jakovetic2014fast,matei2011performance} extended decentralized SGD to other settings, including stochastic networks, constrained problems and noisy environments. Under the non-\iid setting, \citet{tang2018communication} investigated how better performance can be achieved using decentralized SGD. \citet{lian2017can} argued that decentralized algorithms are no compromise for their centralized counterparts, and on par performance can be expected. Our work complements existing studies via generalizing consensus protocols. Additionally, recent analysis on the learning dynamics of deep neural nets suggests distributed learners might be better synced in the nonlinear regime \citep{chizat2019lazy}.

{\bf Worker asynchrony and random graph topology} promise to ameliorate the communication barrier in a distributed setup. Discussion on asynchronous SGD variants and their robustness to communication delays can be found in \citep{zheng2017asynchronous, mitliagkas2016asynchrony, Zhou2018, Mania2015}. Concerning the graph topology \citep{nedic2017achieving, hale2017convergence}, randomized gossiping has been investigated in \citep{hatano2005agreement, boyd2006randomized} and \citet{wang2019matcha} leveraged graph decomposition and randomly sample sub-graphs in each synchronization iteration to improve mixing. Random graph topology has also been considered in the centralized setting \citep{yu2018distributed}. \citet{koloskova2020unified} proposed the decentralized SGD on time-varying random graph. We found our NGO can boost the performance in both cases, furthermore, the randomized connection with a small communication budget under NGO is comparable with the centralized learning.
% \citet{Bellet2017} discussed decentralized learning in a data-privacy setting. Other closely related distributed learning models include federated learning \citep{mcmahan2016communication, Yang2019}, where individual workers perform local updates and sync models in a centralized fashion. Knowledge distillation techniques \citep{hinton2015distilling} also provide a promising future direction for distribution consensus among heterogenous learners. \citet{jiang2018consensus} also discussed consensus and trade-offs in decentralized learning, but from the perspective of tuning the hyperparameters. We provide both theoretical analysis and empirical results to show the proposed NGO can effectively handle both situations. \cytao{This paragraph needs to be cleaned. It is not focused on delays and randomization techniques.}

{\bf Communication efficiency} constitutes one of the most active research topics in distributed learning. There have been myriad of proposals to promote this, for instance by increasing parallelism and increasing computation on
each worker \citep{mcmahan2016communication, zhang2012communication}. Another line of work is to reduce the communication frequency \citep{seide2014one, zinkevich2010parallelized, stich2018local}. Instead of transmitting a full gradient, methods with gradient compression, such as quantization and sparsification, can be a more efficient way to transmit useful information \citep{wen2017terngrad, alistarh2017qsgd, zhang2017zipml, koloskova2019decentralized}. This work investigates an orthogonal problem, which is how to achieve better synchrony under the same communication budget. 

\vspace{-8pt}
\section{Experiments}
\vspace{-6pt}
To validate the utility of the proposed NGO algorithm and benchmark its performance against prior arts, we consider a wide range of synthetic and real-world tasks. All experiments are implemented with PyTorch, and our code is available from \url{https://github.com/author_name/NGO}. Details of our setup \& more results are deferred to the SM Sec. B and C. 

%  In all experiments, we use deep neural networks as the individual learner.
\vspace{-8pt}
\subsection{Experimental setup}
\vspace{-6pt}
% \paragraph{Datasets} We consider the following real-world dataset
%  ($i$) {\texttt{MNIST}} (50K/10K training/test samples)  ($ii$) \texttt{CIFAR10} (50K/10K training/test samples) standard image classification tasks; ($iii$) {\texttt{TINYIMAGENET}} \cite{le2015tiny}: a scaled-down version of the classic natural image dataset \texttt{IMAGENET}, comprised of $200$ classes, $500$ samples per class and $10k$ validation images.
{\bf Datasets.} We consider the following real-world dataset
($i$) {\texttt{MNIST}} ($ii$) \texttt{CIFAR10} standard image classification tasks; ($iii$) {\texttt{TINYIMAGENET}} \citep{le2015tiny}: a scaled-down version of the classic natural image dataset \texttt{IMAGENET}, comprised of $200$ classes, $500$ samples per class and $10k$ validation images.

\vspace{-2pt}
{\bf Data model.} We distribute the $m$ data samples evenly across the $n$ workers, with one of the following two data models: 
\begin{itemize}
	\vspace{-1em}
    \item[$(i)$] \iid setting, where each worker see a \iid copy of data from the same distribution (all labels);
    \vspace{-.5em}
    \item[$(ii)$] Non-\iid setting, where each worker only sees data with a subset of labels. 
\end{itemize} 
\vspace{-1em}
We ensure non-overlapping data assignment, so each data point is assigned to only one worker.
% the partition of data is non-overlapping for all workers, {\it i.e.}, no two workers see a same sample. 
Note that learning with non-\iid data is very challenging.
% , and represents situations where heterogeneous data are locally curated. The detailed data distributions among workers in default settings are shown in the SM Sec. C.
% Specifically for the non-\iid setting in \texttt{MNIST} and \texttt{CIFAR10} we assign every worker samples from exactly $2$ classes of the dataset. In \texttt{TinyImagenet} like previous studies \citep{yurochkin2019bayesian, wang2020federated}, we use Dirichlet distribution to generate the non-\iid data partition among workers. Specifically, we sample $p_k \sim Dir (\beta)$ and allocate a $p_{k,j}$ proportion of the instances of class $k$ to worker $j$, where $Dir(\beta)$ is the Dirichlet distribution with a concentration parameter $\beta$ (0.5 by default). With this partitioning
% strategy, each worker can have relatively few (even no) data samples in some classes. We set the number of parties to 10 by default. 

\vspace{-2pt}
{\bf Baselines.} The following competing baselines are considered to benchmark the proposed solution: $(i)$
centralized learning (parameter server);
$(ii)$ decentralized SGD without compression \citep{lian2017asynchronous},
$(iii)$ Local SGD \citep{stich2018local}.

\vspace{-2pt}
{\bf Network topology.} We consider the following network topologies: ring, random-connected and fully-connected (complete) graphs. 
Unless otherwise specified, we use the ring topology by default for illustrative purposes. This is the most constrained topology, as each worker can only interact with two neighbors. 

\vspace{-2pt}
    {\bf Performance metrics.} We are particularly interested in the level of synchronization, as measured by synchronization index $V$, in addition to the algorithmic convergence in terms of test accuracy and test loss ({\it e.g.}, cross entropy for the supervised learning task and \text{ELBO} for the VI task).

\begin{table*}[t!]
\caption{Comparison of linear and nonlinear Gossiping for supervised learning (Acc) \label{tab:classification}}
\setlength{\tabcolsep}{3pt}
% \begin{threeparttable}[t]
\begin{center}
\begin{small}
\begin{sc}
\begin{tabular}{clccclcc}
\toprule
Dataset & Topology & Gossip & NGO & Dataset & Topology & Gossip & NGO\\
\midrule
\multirow{4}{*}{\iid}
& Centralized &\multicolumn{2}{c}{\textbf{97.47}} &\multirow{4}{*}{Non \iid} &Centralized &\multicolumn{2}{c}{\textbf{83.87}} \\
& Ring & 96.47 & \textbf{96.52}&& Ring & 62.30 & \textbf{64.84}\\
\multirow{2}{*}{\texttt{MNIST}} & Random & 96.48 & \textbf{96.67} &\multirow{2}{*}{\texttt{MNIST}}& Random & 66.72 & \textbf{70.02} \\
& Complete & 96.35 & \textbf{96.70} && Complete & 71.28 & \textbf{73.61}\\
\midrule
\multirow{4}{*}{\iid}
& Centralized &\multicolumn{2}{c}{\textbf{74.52}} &\multirow{4}{*}{Non \iid} & Centralized &\multicolumn{2}{c}{\textbf{70.40}} \\
& Ring & 70.48 & \textbf{72.38} && Ring & 32.49 & \textbf{34.14} \\
\multirow{2}{*}{\texttt{CIFAR10}}& Random & 72.24 & \textbf{73.57} &\multirow{2}{*}{\texttt{CIFAR10}}& Random & 50.11 & \textbf{51.10} \\
& Complete & 72.25 & \textbf{73.42} && Complete & 52.56 & \textbf{53.35}\\
\midrule
\multirow{4}{*}{\iid}
& Centralized &\multicolumn{2}{c}{\textbf{38.39}} &\multirow{4}{*}{Non \iid} & Centralized &\multicolumn{2}{c}{\textbf{36.79}} \\
& Ring & 35.44 & \textbf{36.56}
&& Ring  & 30.67  & \textbf{31.25}\\
\multirow{2}{*}{\texttt{TinyImagenet}}& Random & 37.51 & \textbf{37.84} & \multirow{2}{*}{\texttt{TinyImagenet}}& Random & 31.60 & \textbf{32.83}\\
& Complete & 37.67 & \textbf{37.95} && Complete & 31.93
& \textbf{32.78}\\
% \midrule
% \multirow{4}{*}{\iid}
% & Centralized &\multicolumn{2}{c}{63.65} &\multirow{4}{*}{Non \iid} &Centralized &\multicolumn{2}{c}{0} \\
% & Ring & 59.11 & 60.38 && Ring &  & \\
% \multirow{2}{*}{WikiText2}& Random &62.27 & 63.14 & \multirow{2}{*}{WikiText2}& Random & 46.48 & 48.18\\
% & Complete & 63.25 & 63.64 && Complete & 58.17 & 60.22\\
\bottomrule
\end{tabular}
% \begin{tablenotes}\footnotesize
% \item[*]*\small{random incidence rate is set to 0.4, which guarantees the same communication budget as a ring graph.}
% \end{tablenotes}
\end{sc}
\end{small}
\end{center}
\vspace{-1em}
% \end{threeparttable}
\end{table*}

% \vspace{-8pt}
% \subsection{Comparison to baseline models}
% \vspace{-6pt}

\vspace{-8pt}
\subsection{Impact of nonlinearity}
\vspace{-6pt}
To investigate how the level of nonlinearity (parameter $p$) interacts with learning and synchronization rates to impact synchronization, convergence, and the robustness of training, we plot both synchronization and convergence performance of NGO with varying nonlinearity in Figure \ref{fig:nonlin_para}. As we increase the level of nonlinearity ({\it i.e.,} by decreasing $p$), both performances improve. This is consistent with our theory: NGO enjoys the finite-time consensus property thus synchronizes better with a smaller $p$, which in term reduces the synchronization index discussed in Theorem \ref{centralized}.

\vspace{-8pt}
\subsection{Random graph}
\vspace{-6pt}
We further evaluate NGO with random graphs. We systematically vary the edge probability (Eqn (\ref{aijp})) under the \iid setup, and give the results in Figure \ref{fig:random}. 
% For NGO, the performance improves as $u$ becomes larger, {\it i.e.}, more communications to synchronize. 
After some initial burn-in, the performance of NGO with a rather small $u=0.4$ is on par with that of centralized SGD. We notice that the expectation of a random graph connecting edge numbers with $u=0.2$ is equivalent to a ring topology for 10 workers. Compared to a fixed topology ({\it e.g.}, ring), better performance is observed with the randomized communications with a similar communication budget (See Table \ref{tab:classification} for more results). This can be explained by better information mixing can be expected when workers randomly sync their individual states.
\begin{figure}[t]
	\begin{center}
	\hspace{-1.5em}
	    \includegraphics[width=.5\textwidth]{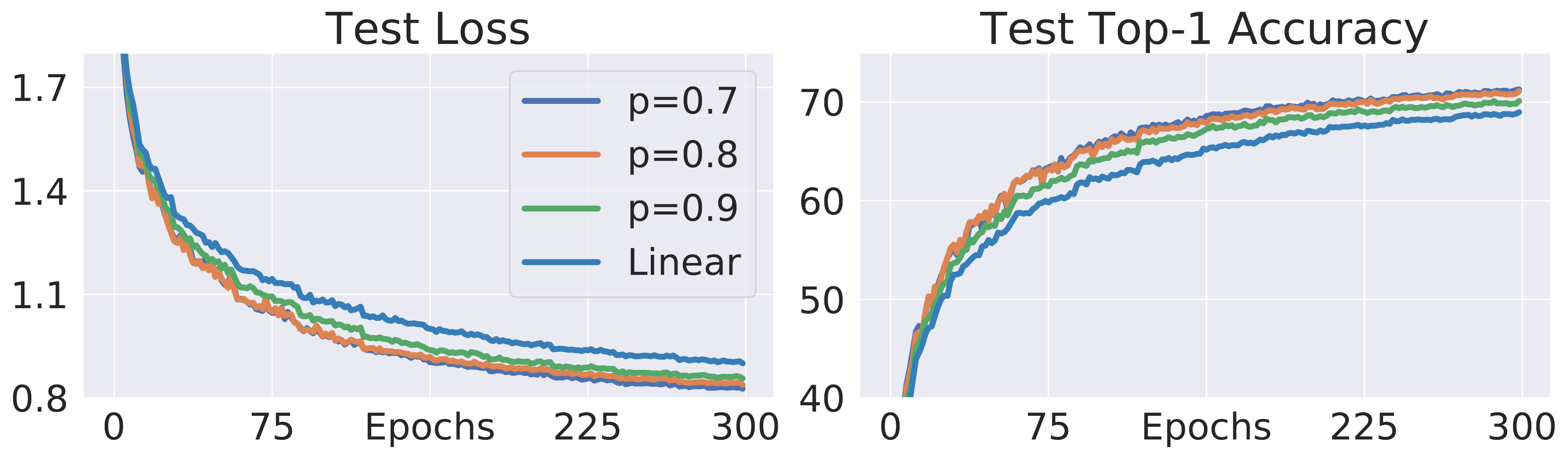}
	    \vspace{-1em}
		\caption{Performance with different choice of NGO parameters.\label{fig:nonlin_para}}
	\end{center}
	 \vspace{-1em}
\end{figure}
\begin{figure}[t]
	\begin{center}
	\hspace{-1.5em}
	    \includegraphics[width=.5\textwidth]{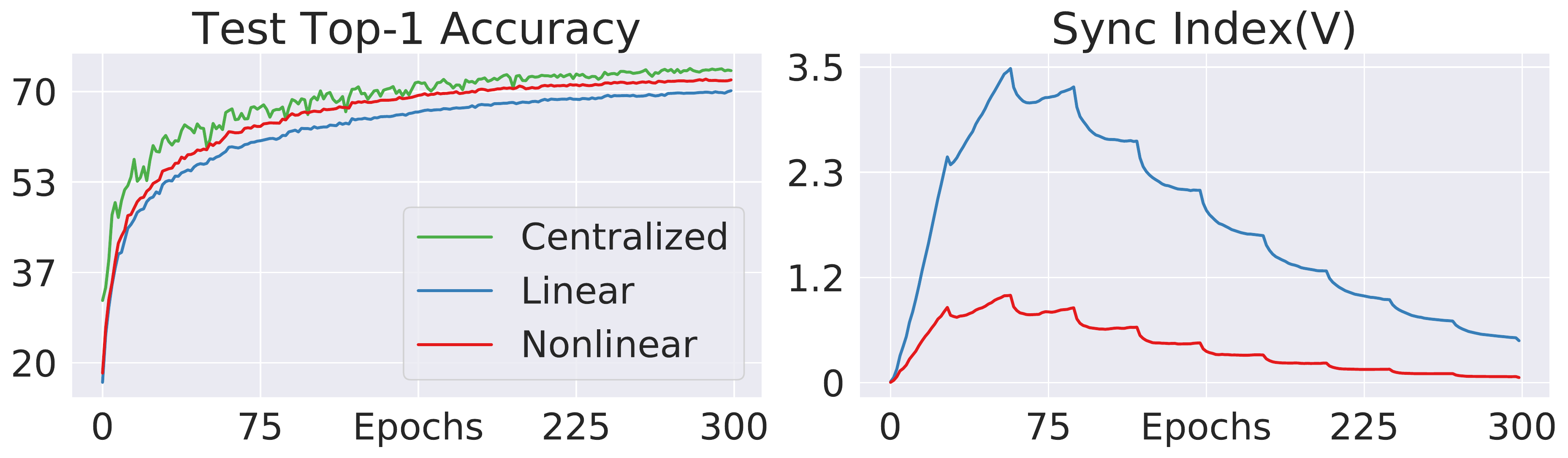}
	    \vspace{-1em}
		\caption{
		Accuracy and synchronization dynamics of linear and nonlinear gossiping and centralized learning on random graph ($u=0.4$). \label{fig:random}}
	\end{center}
	 \vspace{-1.5em}
\end{figure}

% \begin{figure*}[!t]
% \centering
% \begin{minipage}{\textwidth}
% \centering
% 	\includegraphics[width=0.5\textwidth]{aistats_figures/different_p.pdf}
% 	\vspace{-1em}
% 	\caption{Different nonlinear parameter on 10workers.}\label{fig:nonlin_para} 
% \end{minipage}
% \begin{minipage}{\textwidth}
% \centering
%     \includegraphics[width=0.49\textwidth]{aistats_figures/Cifar_results.pdf}
% 	\vspace{-1.5em}
% 	\caption{Convergence of Gossip SGD, NGO SGD and centralized learning on random graph on 10 workers. Note that the models on each worker is initialized as the same model, thus synchronization index starts from zero. \label{fig:random}}
% \end{minipage}
% \end{figure*}
\begin{figure}[t]
	\begin{center}
	\hspace{-1.5em}
	    \includegraphics[width=0.5\textwidth]{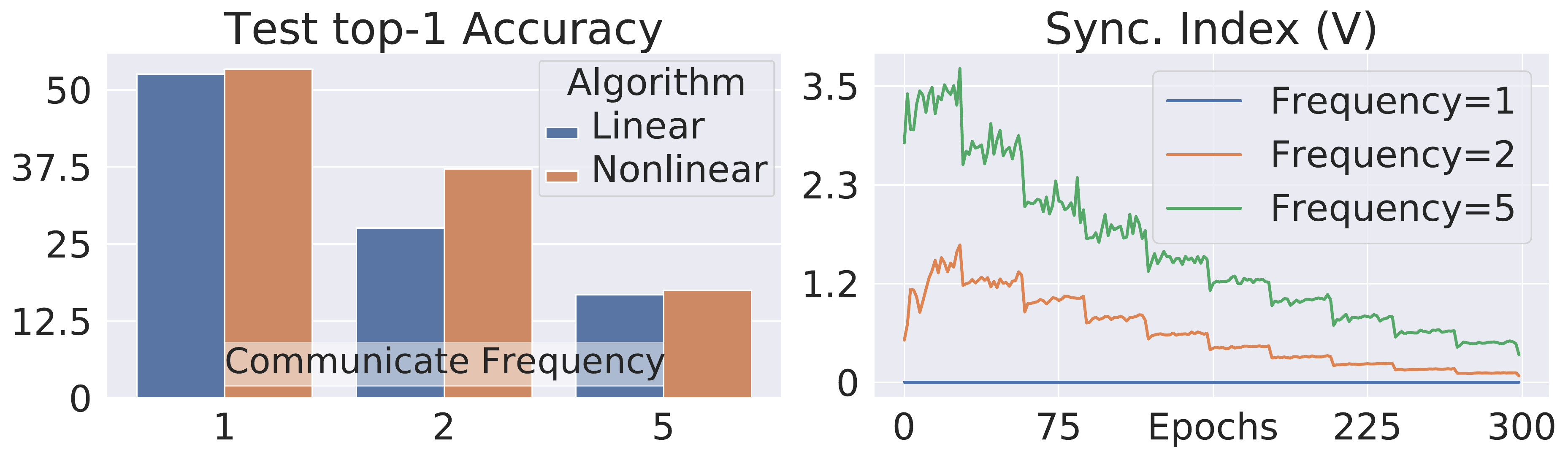}
	    \vspace{-1em}
		\caption{(left) Comparison of local SGD, linear and nonlinear gossiping SGD under the non-\iid setting with complete graph. (right) Synchronization index. \label{fig:localsgd}}
	\end{center}
	 \vspace{-2em}
\end{figure}

\vspace{-8pt}
\subsection{Communication and convergence}
\vspace{-6pt}
We offer an alternative view that to alleviate the communication burden in distributed learning, one should optimize the communication efficiency rather than reducing the communication frequency (which is a practice in local SGD \citep{stich2018local}), especially so for the non-\iid setting. To confirm this, we reduce the consensus update frequency to implement a local SGD variant, and compare its performance with regular implementations. Specifically, we set the synchronization frequency set as the multiple of 1 (gossip SGD), 2 and 5 (local SGD), {\it i.e.,} workers communicate with their neighbors every 1, 2 and 5 iterations of SGD. In Figure \ref{fig:localsgd}, we see that local SGD fails to match the performance with their chatty counterparts. This issue is less severe in the \iid setting, since each worker updates its parameters towards an approximately correct direction. However, slower convergence can be expected (see SM Sec. C). 
% We further observe the strong correlation between synchronization index with test loss in Figure \ref{fig:sync_loss}(left).
% localSGD suffers from a large variance due to the lack of communication and the non-\iid setting aggravates the need of communication, as a result, it cannot converge to a low training loss.

% We use non-\iid setting to increase the difficulty of the task. We plot the logarithm of the Lyapunov function to depict the synchronization of Gossip SGD and NGO SGD, and the performance of those two algorithms and LocalSGD. 
%  Figure \ref{fig:localsgd} shows that local SGD suffers from a large variance due to inadequate communication and the non-\iid setting aggravates this issue. 

% In Figure \ref{fig:sync_loss} we observe that high level of synchronization helps

\vspace{-8pt}
\subsection{Scaling to large number of workers}
\vspace{-6pt}
To study the scaling properties of NGO, we train on 4, 16, 36 and 64 number of nodes. We give every worker the same amount of data regardless of the network size. The results are summarized in (Figure \ref{fig:scaling}). Centralized learning has a good performance for all scales of networks. For decentralized protocols, our NGO consistently outperforms linear gossiping. Better performance is expected with the increase of workers as more data is utilized in training.

\begin{figure}[t]
	\begin{center}
		\includegraphics[width=0.5\textwidth]{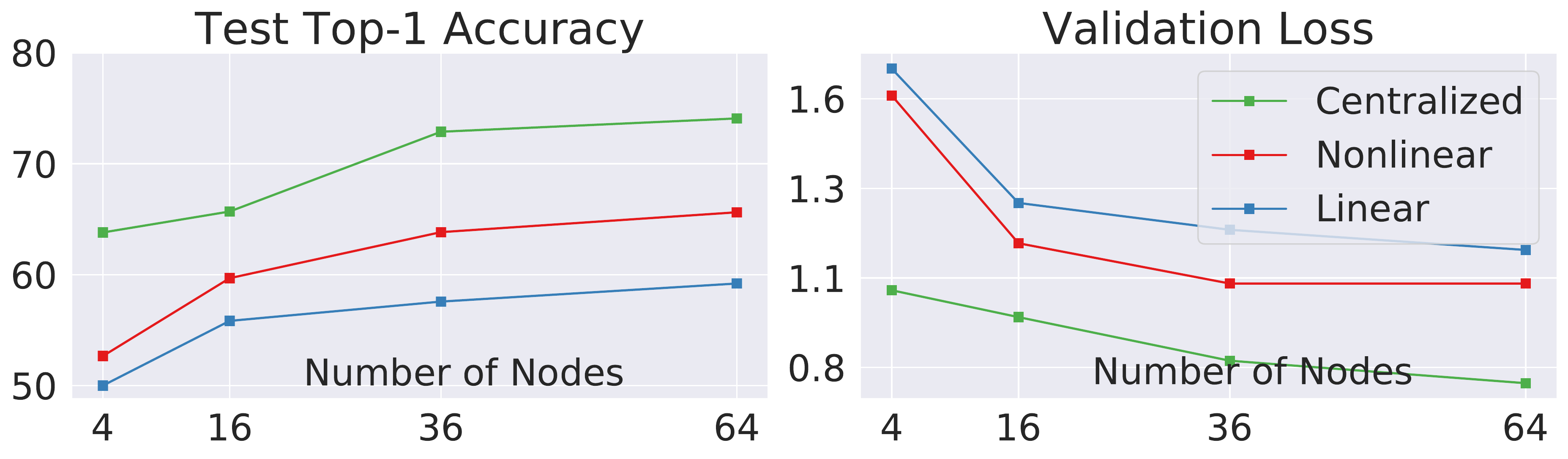}
		\vspace{-2em}
		\caption{Scaling of NGO to large number of workers. Best testing accuracy of the algorithms reached after 300 epochs.\label{fig:scaling}}
	\end{center}
	 \vspace{-1em}
\end{figure}

\begin{figure}[t]
	\begin{center}
		\includegraphics[width=0.5\textwidth]{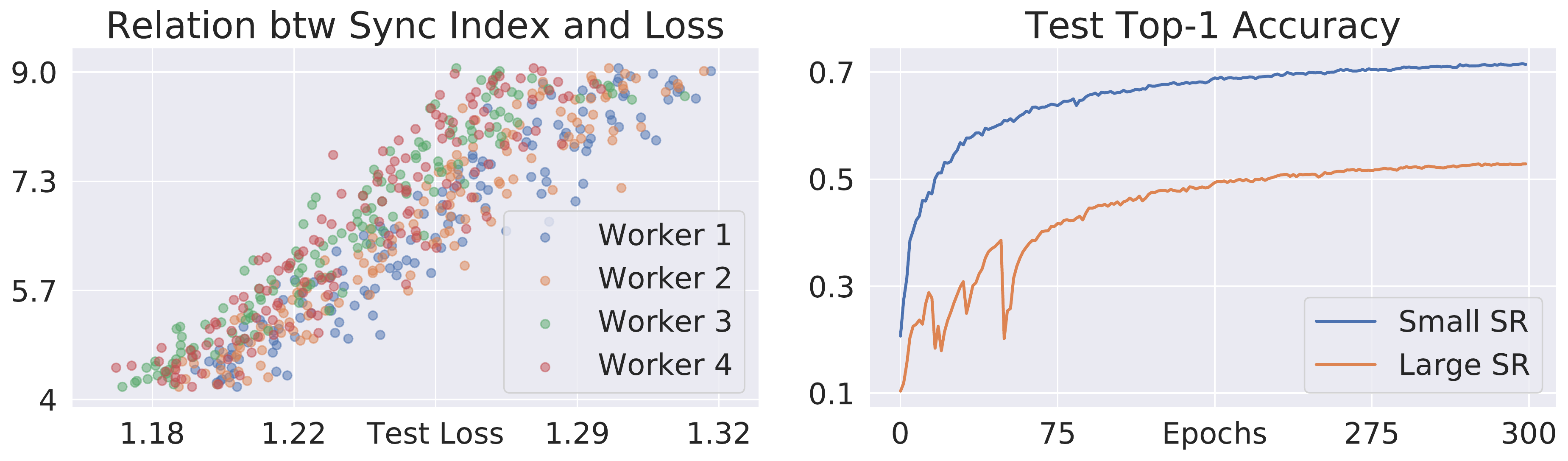}
		\vspace{-2em}
		\caption{(Left) Worker synchronization strongly correlates with performance. (Right) Nonlinear gossiping suffers performance drop at larger synchronization rate.\label{fig:sync_loss}}
	\end{center}
	 \vspace{-1em}
\end{figure}

\vspace{-8pt}
\subsection{Synchronization rate and stability}
\vspace{-6pt}
While in general NGO delivers more favorable results compared to its linear counterparts, we notice NGO should be practiced with caution. This is because NGO can be more sensitive to the synchronization rate $\gamma_t$. When $\gamma_t$ is too high, some oscillations is expected for NGO (Figure \ref{fig:sync_loss} right). 

\vspace{-8pt}
\section{Conclusions}
\vspace{-6pt}
We present a novel nonlinear gossiping SGD framework based on finite-time consensus protocol. We prove that NGO achieves a better convergence rate compared to linear gossiping, bridging the gap of centralized learning and decentralized learning. We highlight that our algorithm achieves better performance under a limited communication budget compared to gossiping SGD. Unlike prior-arts, we seek to improve the efficiency of communication with respect to synchronization rather than reducing the communication frequency or compressing messages. NGO can be easily adapted to work with asynchronous communications and randomized graph topology. In future work, we wish to establish NGO's convergence under non-convex settings, and extend its applicability to heterogeneous learners. 
\bibliography{ngo}

\newpage

\appendix

\renewcommand{\thetable}{S\arabic{table}}
\renewcommand{\thefigure}{S\arabic{figure}}
\renewcommand{\thealgorithm}{S\arabic{algorithm}}
\renewcommand{\thethm}{S\arabic{thm}}

\setcounter{table}{0}
\setcounter{figure}{0}
\setcounter{algorithm}{0}
\setcounter{thm}{0}

%%%%% To create table of contents only for Appendix %%%%%%
\onecolumn
\addcontentsline{toc}{section}{Appendix} % Add the appendix text to the document TOC
\part{Appendix} % Start the appendix part
\parttoc % Insert the appendix TOC
%%%%%%%%%%%%%%%%%%%%%%%%%%%%%%%%%%

\section{Technical Support}
\subsection{Theoretical Results of Synchronization}
Before stating the theoretical results, we need to define an important class of digraphs that appear
frequently throughout this section.
\begin{defn}[balanced digraphs \cite{olfati2004consensus}]
A digraph $\mathcal{G}$ is
called balanced if $\sum_{j\neq i} \Wv_{ij} = \sum_{j\neq i} \Wv_{ji}$ for all $i\in \mathcal{V}$.
\end{defn}
We first revisit the continuous version of gossip algorithm.
\begin{lem}[\cite{olfati2007consensus}]
    Let $\mathcal{G}$ be a connected undirected graph. Then, 
    \begin{equation}\label{dyn}
        \dot{\xv_i} = \sum_{j:(i,j)\in\mathcal{E}} \Wv_{ij}(\xv_j(t)-\xv_i(t))
    \end{equation}
   asymptotically reaches synchronization for all initial states. Moreover, the synchronization value is $\frac{1}{n} \sum_i  \xv_i(0)$ that is equal to the average of the initial values.
\end{lem}
\begin{proof}
    The dynamics of system (\ref{dyn}) can be expressed in a compact form as 
    \begin{equation}
        \dot{\xv} = -\mathbf{L} \xv
    \end{equation}
    where $\mathbf{L}$ is known as the graph Laplacian of $\mathcal{G}$. By definition, $\mathbf{L}$ has a right eigenvector of $\mathbf{1}$ associated with
    the zero eigenvalue because of the identity $\mathbf{L} \mathbf{1} = \mathbf{0}$. For the case of undirected graphs, graph Laplacian satisfies the following property:
    \begin{equation}\label{eq:sos}
        \xv^\top \mathbf{L}\xv=\frac{1}{2}\sum_{(i,j)\in \mathcal{E}} \Wmat_{ij}(\xv_j-\xv_i)^2
    \end{equation}
    By defining a quadratic disagreement function as
    \begin{equation}
        \phi(\xv) = \frac{1}{2}\xv^\top \mathbf{L}\xv
    \end{equation}
    it becomes apparent that algorithm (\ref{dyn}) is the same as 
    \begin{equation}
        \dot{\xv} = -\nabla\phi(\xv)
    \end{equation}
    or the gradient-descent algorithm. This algorithm globally asymptotically converges to the agreement space provided that two conditions hold: 1) $\mathbf{L}$ is a positive semidefinite
matrix; 2) the only equilibrium of (\ref{dyn}) is $\alpha \mathbf{1}$ for some $\alpha$. Both of these conditions hold for a connected graph and follow from the SOS property of graph Laplacian in (\ref{eq:sos}). Therefore, an average-synchronization is asymptotically reaches for all initial states. 
\end{proof}
The discrete-time convergence result is almost
identical to its continuous-time counterpart.
\begin{thm}[\citep{ren2005consensus}]
Consider a network of workers with topology $\mathcal{G}$ applying the distributed
consensus algorithm
\begin{equation}
    \xv_i^{(t+1)} = \xv_i^{(t)}+\gamma \sum_{j:(i,j)\in \mathcal{G}} \Wmat_{ij}(\xv_j^{(t)} -\xv_i^{(i)})
\end{equation}
where $0<\gamma<1/\Delta$, and $\Delta$ is the maximum degree of the network. Let $\mathcal{G}$ be a strongly connected digraph. Then an average-synchronization is asymptotically reached if \Wmat is doubly stochastic.
\end{thm}

Suppose that worker $i$ receives a message sent by its
neighbor $j$ after a time-delay of $\tau$. This is equivalent to a network with a uniform one-hop communication time delay. 
\begin{thm}
The algorithm 
\begin{equation*}
    \dot{\xv}(t) = \sum_{j:(i,j)\in \mathcal{G}}\Wmat_{ij}(\xv_j(t-\tau)-\xv_i(t-\tau))
\end{equation*}
asymptotically solves the average synchronization problem with a uniform one-hop time-delay $\tau$
for all initial states if and only if $0\leq \tau<\pi/2\lambda_n$, where $\lambda_n$ is the largest eigenvalue of the graph Laplacian $\mathbf{L}$.
\end{thm}

A switching network can be modeled using a dynamic graph $\mathcal{G}_{s_t}$ parameterized with a switching signal $s_t:\mathbb{R}\rightarrow J$ that
takes its values in an index set $J =\{1,2,\cdots, n\}$.
\begin{thm}[\cite{olfati2004consensus}]
Consider a network of workers with algorithm $\dot{\xv} = \mathbf{L}_{\mathcal{G}_k}$ with topologies ${\mathcal{G}_k}\in \Gamma$. Suppose every graph in $\Gamma$ is a balanced digraph that is strongly connected and let $\lambda_2^* = \min_{k\in J}\lambda_2(\mathcal{G}_k)$. Then, for any arbitrary switching signal, the workers asymptotically reach an average-synchronization for all initial states with a speed faster or equal to $\lambda_2^*$.
\end{thm}

% \begin{lem}[\cite{olfati2007consensus}]
%     Let $\mathcal{G}$ be a connected undirected graph. Then, 
%     \begin{equation}
%         \dot{\xv_i} = \sum_{j:(i,j)\in\mathcal{E}} \Wv_{ij}(\xv_j(t)-\xv_i(t))
%     \end{equation}
%   asymptotically reaches synchronization for all initial states. Moreover, the synchronization value is $\frac{1}{n} \sum_i  \xv_i(0)$ that is equal to the average of the initial values.
% \end{lem}
\subsection{Proof of Theorem 3.1}
\begin{proof}
Recall the definition of Lyapunov function
\begin{equation*}
    V(t) = \sum_{i=1}^n\left(\xv_i(t) - \bar{\xv}\right)^2,
\end{equation*}
after differentiating we have:
\begin{equation}
	\begin{split}
	\dot{V}(t) &= 2\sum_{i}\left(\xv_i(t) - \bar{\xv}\right)(\dot{\xv}_i(t) - \dot{\bar{\xv}})\\
	&= 2\gamma\sum_i \left(\xv_i(t) - \bar{\xv}(t)\right)\sum_{j:(i,j)\in \mathcal{E}} \Wmat_{ij}\phiv\left(\xv_j(t) - \xv_i(t)\right)
	\end{split}
\end{equation}

Here we use the fact that $\sum_{(i,j)\in \mathcal{E}}\Wmat_{ij}\phiv(\xv_j(t)-\xv_i(t)) = 0, \forall t>0$, thus $\bar{\xv}(t) = \bar{\xv}$ is invariant.

\begin{equation}
	\begin{split}
	&\sum_{i}\sum_{j:(i,j)\in \mathcal{E}} \Wmat_{ij}\left|\xv_i(t) -\xv_j(t)\right|\phiv\left(\left|\xv_j(t) - \xv_i(t)\right|\right)\\
	=& \sum_{(i,j)\in \mathcal{E}}\Wmat_{ij}\left(\xv_i(t) - \xv_j(t)\right)^{2p}\\
	\geq & \left[\sum_{(i,j)\in\mathcal{E}} \Wmat_{ij}^{\frac{1}{p}}\left(\xv_i(t) - \xv_j(t)\right)^2 \right]^{p}\\
	=& \left[\sum_{(i,j)\in\mathcal{E}} \Wmat_{ij}^{\frac{1}{p}}\left(\deltav_i(t) - \deltav_j(t)\right)^2 \right]^{p}\\
	%&\geq \underline{a} \sum_{ij} \mu\left(\sum_{ij}|x_i(t) - x_j(t)|^2\right)
	\end{split}
\end{equation}
where $\Bmat =\left[\Wmat_{ij}^{\frac{1}{p}}\right]\in \BR^{n\times n}$. Notice that $\sum_{i,j=1}^{n} \Wmat_{ij}^\frac{1}{p}(\deltav_i-\deltav_j)^2 = 2\deltav^\top \Lmat(\Bmat)\deltav$, and $\sum_i \deltav_i = 0$,
we have
\begin{equation*}
	\frac{\sum_{i,j = 1}^n \Wmat_{ij}^{\frac{1}{p}} \left(\deltav_i-\deltav_j\right)^2}{V(t)} = \frac{2\deltav^\top \Lmat(\Bmat)\deltav}{\deltav^\top\deltav} \geq 2\lambda_2(\Lmat(\Bmat)).
\end{equation*}
where $\Lmat(\Bmat)$ is the graph Laplacian of $\mathcal{G}(\Bmat)$ and $\lambda_2(\Lmat(\Bmat))$ is the algebraic connectivity of $\mathcal{G}(\Bmat)$, $\lambda_2(\Lmat(\Bmat))$ is positive since the graph is connected.

Therefore,
\begin{equation}
\begin{split}
	\dot{V}(t) &= 2\gamma\sum_i(\xv_i(t)-\bar{\xv})\sum_{j:(i,j)\in \mathcal{E}} \Wmat_{ij}\phiv(\xv_j(t)-\xv_i(t))\\
	&=2\gamma\sum_{(i,j)\in \mathcal{E}} \Wmat_{ij}\xv_i(t)\phiv(\xv_j(t) - \xv_i(t))\\
	&=2\gamma\sum_{(i,j)\in \mathcal{E}} \Wmat_{ij}(\xv_i(t) - \xv_j(t))\phiv(\xv_j(t) - \xv_i(t))\\
	&=-2\gamma\sum_{(i,j)\in \mathcal{E}} \Wmat_{ij}|\xv_i(t) -\xv_j(t)|\phiv \left(|\xv_j(t) - \xv_i(t)|\right)\\
	&\leq -4\gamma \left(\lambda_2(\Lmat(\Bmat))\right)^p V(t)^{p}
	\end{split}
\end{equation}

Therefore $V(t) = 0$, when $$t>T^* = \frac{V^{1-p}(0)}{4\gamma(1-p) \left(\lambda_2(\Lmat(\Bmat))\right)^p}$$.
\end{proof}
\subsection{Convergence rate}
The following lemma for decentralized learning is derived in \cite{koloskova2019decentralized}
\begin{lem}
	The average $\bar{\xv}^{(t)}$ of the iterates of  Algorithm 2 and Algorithm 3 satisfy the following:
	\begin{equation*}
	\begin{split}
	&\mathbb{E}_{\xi_1^{(t)}, \cdots, \xi_n^t}\left\|\bar{\xv}^{(t+1)} -\xv^*\right\|^2\\
	\leq& \left(1-\frac{\mu\eta_t}{2}\right)\left\|\bar{\xv}^{(t)}-\xv^*\right\|^2 +\frac{\eta_t^2\bar{\sigma}^2}{n} 
	-2\eta_t\left(1-2L\eta_t\right)\left(f(\bar{\xv}^t) - f^*\right) \\
	+&\frac{2 \eta_t^2L^2+L\eta_t+\mu\eta_t}{n}\sum_{i=1}^{n}\|\bar{\xv}^{(t)} - \xv_i^{(t)}\|^2
	\end{split}
	\end{equation*}
	where $\bar{\sigma}^2 = \frac{1}{n}\sum_{i=1}^{n}\sigma_i^2$.
\end{lem}

Similarly, for centralized learning, we have
\begin{equation*}
	\mathbb{E}_{\xi_1^{(t)}, \cdots, \xi_n^{(t)}}\left\|\xv^{(t+1)} - \xv^*\right\|^2\leq \left(1-\frac{\mu\eta_t}{2}\right)\left\|\xv^{(t)}-\xv^*\right\|^2 +\frac{\eta_t^2\bar{\sigma}^2}{n} -2\eta_t\left(1-2L\eta_t\right)\left(f\left(\xv^{(t)}\right) - f^*\right)
\end{equation*}

According to Lemma 21 in \cite{koloskova2019decentralized}, for linear gossip algorithm we have:
\begin{equation}\label{last_term}
	\left\|{\xv}^{(t+1)} - \bar{\xv}^{(t+1)}\right\|_F^2 \leq 40\eta_t^2\frac{1}{\beta^2}nG^2
\end{equation}
where $\beta =1-\gamma\lambda_2(\Lmat(\Wmat)) $. For nonlinear gossip algorithm, according to Theorem 3.1, the term \eqref{last_term} vanishes in finite time $T^*$.

\begin{thm}
With stepsize $\eta_t =\frac{4}{\mu(a+t)}$, for parameter $a \geq \max\{\frac{5}{\beta}, 15 \kappa\}, \kappa=\frac{L}{\mu}$, centralized learning converges at the rate 
$$f\left(\xv^{(T)}_{\text{avg}}\right)-f^*\leq \frac{\mu a^3}{8S_T}\|\bar{\xv}^{(0)}-\xv^*\|^2+\frac{4T(T+2a)}{\mu S_T}\frac{\bar{\sigma}^2}{n}$$
nonlinear gossip learning converges at the rate 
$$f\left(\xv^{(T)}_{\text{avg}}\right)-f^*\leq \frac{\mu a^3}{8S_T}\|\bar{\xv}^{(0)}-\xv^*\|^2+\frac{4T(T+2a)}{\mu S_T}\frac{\bar{\sigma}^2}{n} + \frac{2L+\mu}{S_T}\sum_{t=0}^{T^*}\sum_{i=1}^n \frac{(a+t)^2}{n} \|\xv_i^{(t)}-\bar{\xv}\|^2$$
where $\xv^{(T)}_{\text{avg}} = \frac{1}{S_T}\sum_{t=0}^{T-1}w_t\bar{\xv}^{(t)} 
$ for weights $w_t = (a+t)^2$, and $S_T= \sum_{t=0}^{T-1}w_t\geq \frac{1}{3}T^3$, finite time 
\begin{equation*}
    T^* = \frac{\left(\sum_{i}\left(\xv_i^{(0)}-\bar{\xv}^{(0)}\right)^2\right)^{1-p}}{4\gamma (1-p) \left(\lambda_2(\Lmat(\Bmat))\right)^p },
\end{equation*} 
with $\Bmat =\left[\Wmat_{ij}^{\frac{1}{p}}\right]$,  $\lambda_2(\Lmat(\Bmat))$ is the algebraic connectivity of $\mathcal{G}(\Bmat)$.
\end{thm}

\begin{proof}
	For $\eta_t\leq \frac{1}{4L}$ it holds $2L\eta_t - 1\leq -\frac{1}{2}$ and $(2\eta_tL^2+L+\mu)<(2L+\mu)$, hence
	\begin{equation}
		\mathbb{E}\|\xv^{(t+1)} - \xv^*\|^2 \leq \left(1-\frac{\mu\eta_t}{2}\right)\mathbb{E}\|\xv^{(t)}- \xv^*\|^2 +\frac{\eta_t^2\bar{\sigma}^2}{n} -\eta_t e_t +\eta_t \frac{2L+\mu}{n} V^{(t)}
	\end{equation}
	
	From Lemma 23 in \cite{koloskova2019decentralized}, in nonlinear gossip algorithm we get 
	\begin{equation*}
	\frac{1}{S_T}\sum_{t=0}^{T-1}w_te_t\leq \frac{\mu a^3}{8S_T}a_0 +\frac{4T(T+2a)}{\mu S_T}\frac{\bar{\sigma}^2}{n}+\frac{2L+\mu}{nS_T }\sum_{t=0}^{T^*}V^{(t)}
	\end{equation*}
	for weights $w_t =(a+t)^2$ and $S_T\triangleq \sum_{t=0}^{T-1}w_t = \frac{T}{6}\left(2T^2+6aT-3T+6a^2 -6a+1\right)\geq \frac{1}{3}T^3$. The last term equals zero in centralized learning.
\end{proof}

\section{Toy Model Experiments}
We set the target model as $y=\sin x +\epsilon$, where $\epsilon\sim N(0,0.1)$ is a random noise added to the data generation. In the toy model, the 6400 datapoints randomly assigned to 5 workers, and the partition of data is non-overlapping for all the workers.
The network on each worker is set as a 3-layer linear network: $1\xrightarrow{\text{Linear}+ \text{ReLU}} 64\xrightarrow{\text{Linear}+ \text{ReLU}}
64\xrightarrow{\text{Linear}+ \text{ReLU}}1$.

\section{Real-world Experiments}
\subsection{Data Partition}
We distribute the $m$ data samples evenly across the $n$ workers, with one of the following two data models: 
\begin{itemize}
    \item[$(i)$] \iid setting, where data are randomly assigned to each worker, such that each worker sees \iid copies from the same distribution;
    \item[$(ii)$] Non-\iid setting, where each worker only sees data with one, or a few, of particular labels. 
\end{itemize} 
Specifically, under the non-\iid setting, we split {\texttt{MNIST}} into 300 shards, and each worker takes two shards of images from the data pool without replacement. In non-\iid {\texttt{CIFAR10}}, the data is splitted into 40 shards and each worker selects 4 shards from the pool without replacement. 
We use Dirichlet distribution to generate the non-\iid {\texttt{TINYIMAGENET}} data partition among workers \citep{xiu2020supercharging}.
We sample $p_k \sim Dir_{N} (\beta)$ and allocate a 
$p_{k, j}$ proportion of the instances of class $k$ to worker $j$, where $Dir(\beta)$ is the Dirichlet distribution with a concentration parameter $\beta$ (0.5 by default). The data distributions among workers in default settings are shown in Figure \ref{fig:class_id}.
\begin{figure}[t]
	\begin{center}\label{fig:corr}
		\includegraphics[width=.7\textwidth]{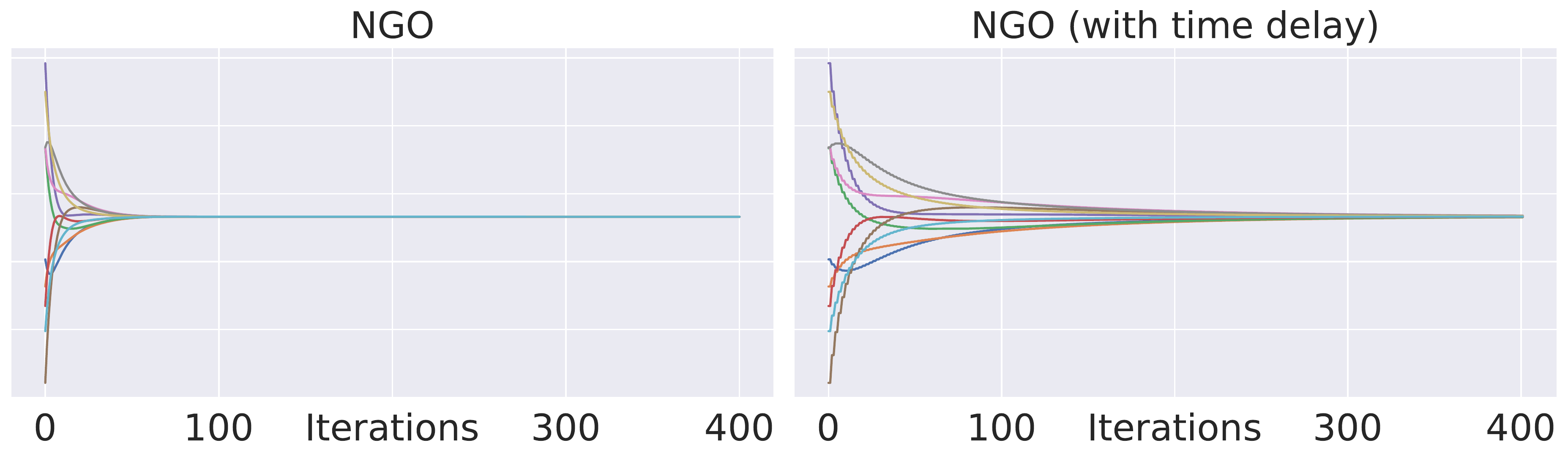}
		\caption{Consensus result on nonlinear gossip with and without time delay.}
	\end{center}
\end{figure}

\begin{figure}[t]
	\begin{center}\label{fig:class_id}
		\includegraphics[width=\textwidth]{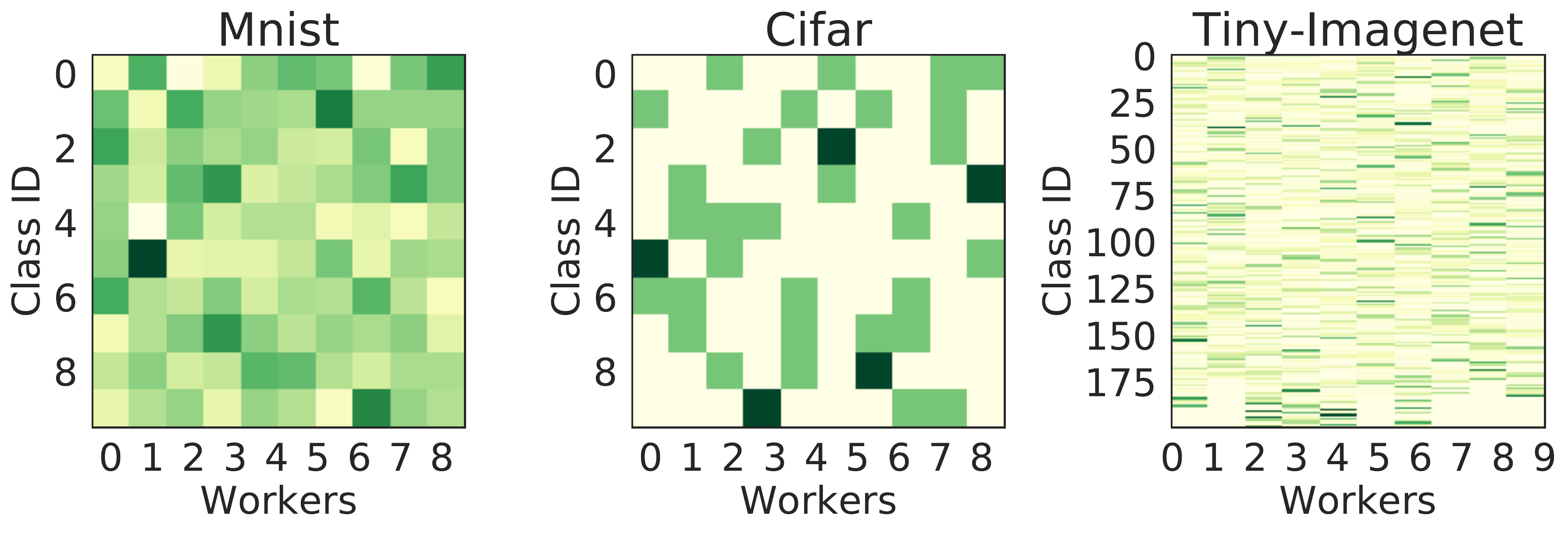}
		\caption{The data distribution of each worker using non-\iid data partition. The color bar denotes the number of data samples. Each rectangle represents the number of data samples of a specific class on a worker.}
	\end{center}
\end{figure}

\subsection{Supervised Learning}
\paragraph{Scalability}
To further study the scaling properties of NGO, we train on 4, 16, 36 and 64 number of nodes. We notice that under this setting, we distribute the whole dataset evenly to nodes, which is different from Sec. 5.5.
The corresponding parameters are listed in Table \ref{tab:spectral}.
\begin{table}[t!]
\caption{Summary of communication topologies \label{tab:spectral}}
\setlength{\tabcolsep}{3pt}
\begin{center}
\begin{small}
\begin{sc}
\begin{tabular}{lccccc}
\toprule
\multirow{2}{*}{Toplogy}&  & \multicolumn{4}{c}{spectral gap}\\\cmidrule(lr){2-2} \cmidrule(lr){3-6}
 & node degree & $n=4$ & $n=16$ & $n=36$ & $n=64$\\
\midrule
ring & 2 & 0.67 &0.05&0.01&0.003 \\
random($u$) & $un$& - & -&-&-\\
complete &$(n-1)$&1&1&1&1\\
\bottomrule
\end{tabular}
\end{sc}
\end{small}
\end{center}
\end{table}
Gossip type algorithm slows down due to the influence of the graph topology which is consistent with the spectral gaps order (see Tab. \ref{tab:spectral}). Nonlinear gossip SGD consistently outperforms linear gossip, and the improvement is more obvious when number of nodes increases.
\begin{figure}[t]
	\begin{center}\label{fig:generate}
		\includegraphics[width=\textwidth]{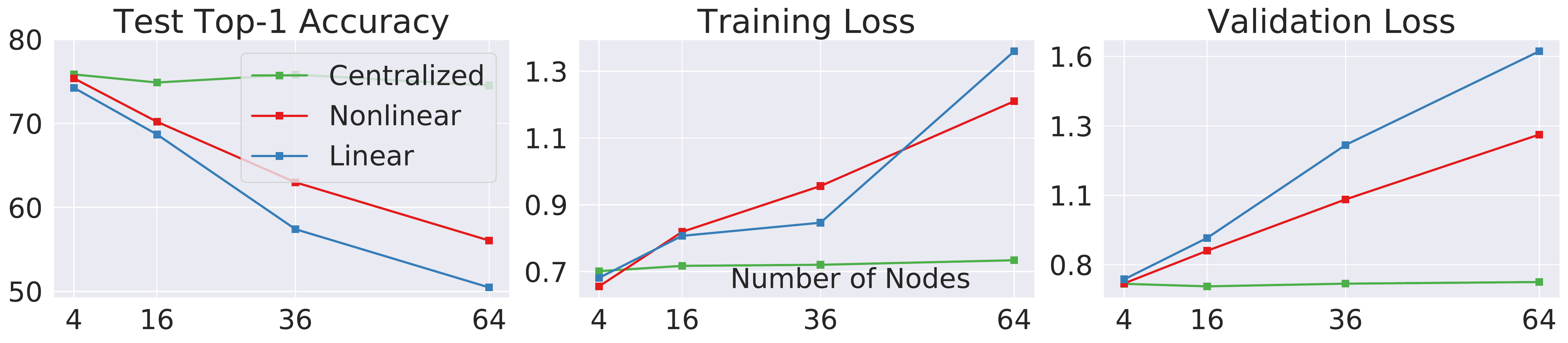}
		\caption{Scalability results on \iid CIFAR10}
	\end{center}
\end{figure}

\begin{figure}[t]
	\begin{center}\label{fig:corr}
		\includegraphics[width=.4\textwidth]{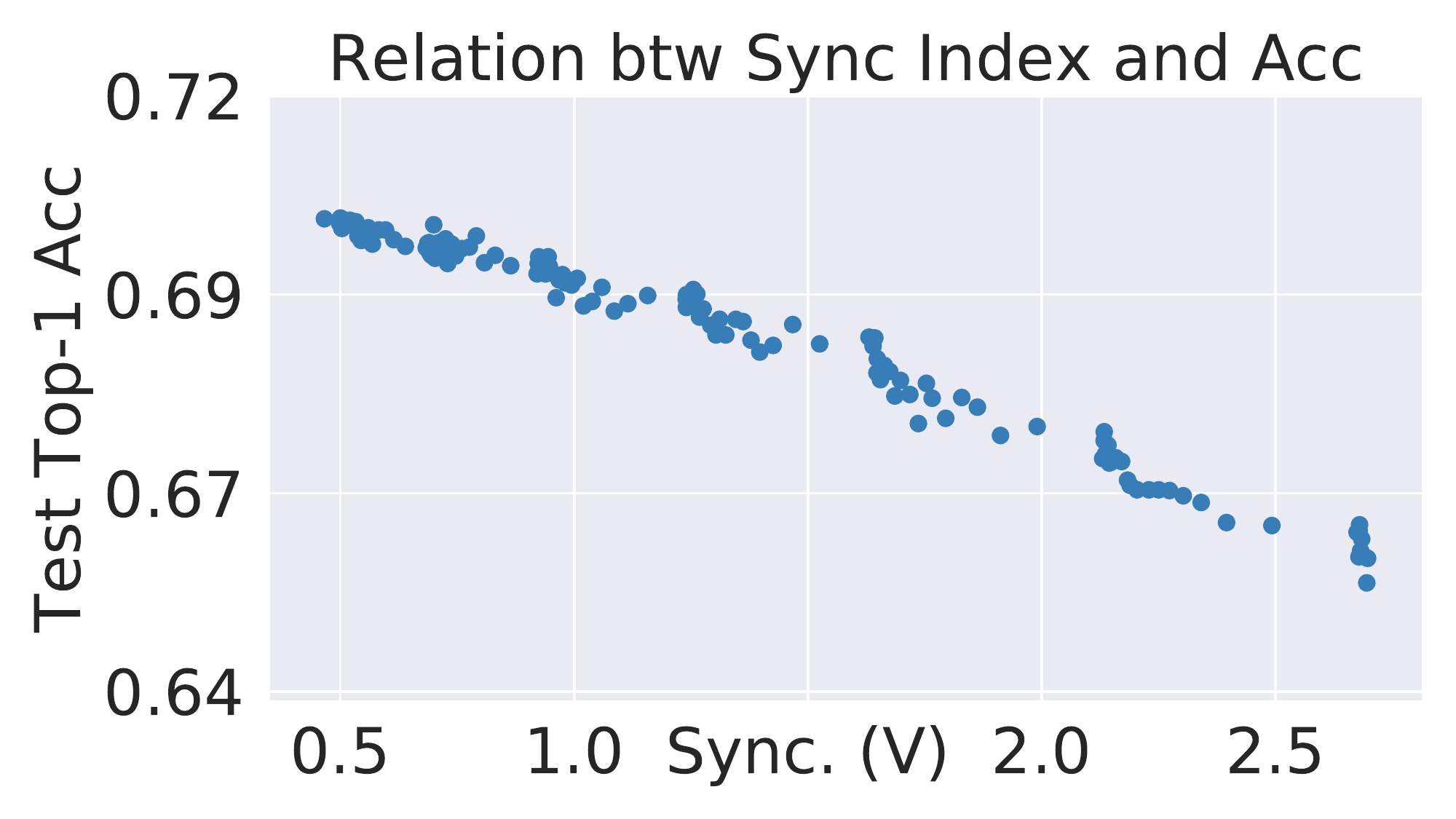}
		\caption{(left) Correlation between synchronization and test accuracy. }
	\end{center}
\end{figure}

% \paragraph{Random graph}

\subsection{Unsupervised Learning}
We now compare the performance of linear and nonlinear gossiping for unsupervised learning problem. Here we adopt the VAE setting, where one seeks to maximize the lower bound to the data likelihood by optimizing the stochastic autoencoder $q(z)$ and conditional data likelihood $p(x|z)$ \citep{chen2021variational,tao2019fenchel}. We compare the result of linear and nonlinear gossiping, together with the centralized learning result under \iid setting and non-\iid setting. In Table \ref{table1}, we compare the evidence lower bound on MNIST dataset and the results show that NGO outperforms Gossiping under both settings, and NGO performs almost as good as centralized learning. For Cifar 10 and MNIST, we provide a few reconstruction with the average model trained with NGO compared to their samples in Figure \ref{fig:generate}.
\begin{figure}[t]
	\begin{center}
		\includegraphics[width=0.45\textwidth]{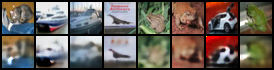}
		\includegraphics[width=0.45\textwidth]{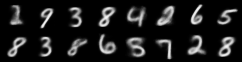}
		\caption{(left) Samples reconstructed from \texttt{CIFAR10} dataset with average model trained with nonlinear gossip algorithm. (right) Samples generated from \texttt{MNIST} dataset with model on different workers trained with nonlinear gossip algorithm. \label{fig:generate}}
	\end{center}
\end{figure}

\subsection{Network Achitectures}
\begin{table}[h]\label{table1}
\caption{Comparison of NGO and Gossip for unsupervised learning (VAE) \label{tab:vae}}
\setlength{\tabcolsep}{3pt}
\begin{center}
\begin{small}
\begin{sc}
\begin{tabular}{lcccccccccccc}
\toprule
Setting& \multicolumn{3}{c}{Gossip} & \multicolumn{2}{c}{NGO} \\
\midrule
\iid Centralized&\multicolumn{5}{c}{113.73}\\
\iid Ring&\multicolumn{3}{c}{124.35}&\multicolumn{2}{c}{113.27}\\
\iid Random&\multicolumn{3}{c}{124.83}&\multicolumn{2}{c}{114.62} \\
\iid FC &\multicolumn{4}{c}{121.79}&\multicolumn{2}{c}{114.41}\\
\midrule
non\ \iid Centralized&\multicolumn{5}{c}{93.18}\\
non\ \iid Ring&\multicolumn{3}{c}{113.05}&\multicolumn{2}{c}{111.56}\\
non\ \iid Random&\multicolumn{3}{c}{113.37}&\multicolumn{2}{c}{110.59}\\
non\ \iid FC &\multicolumn{4}{c}{113.43}&\multicolumn{2}{c}{110.82}\\
\bottomrule
\end{tabular}
\end{sc}
\end{small}
\end{center}
\end{table}

\begin{table}[!t]
\begin{minipage}{.55\linewidth}
\caption{MNIST classication experiment network architecture.\label{tab:mnist_net}}
\setlength{\tabcolsep}{25pt}
\centering
\begin{sc}
\begin{small}
\resizebox{.9\textwidth}{!}{\begin{tabular}{cl}
\toprule
\multicolumn{1}{c}{Network} & \multicolumn{1}{c}{Architecture} \\
\midrule
Classifier & conv2d(unit=10, kernel=5, stride=1)\\
&+maxpool2d(stride=2)+ReLU \\
& +dropout2d(p=0.5)\\
& +conv2d(unit=20, kernel=5, stride=1)\\
&+maxpool2d(stride=2)+ReLU \\
&+ fc(unit=50)+relu\\
& +dropout2d(p=0.5)\\
&+ fc(unit=10)+logsoftmax\\
\bottomrule
\end{tabular}}
\end{small}
\end{sc}
\end{minipage}
\begin{minipage}{.43\linewidth}
\caption{MNIST VAE experiment network architecture.\label{tab:mnist_net}}
\setlength{\tabcolsep}{25pt}
\centering
\begin{sc}
\begin{small}
\resizebox{.9\textwidth}{!}{\begin{tabular}{cl}
\toprule
\multicolumn{1}{c}{Network} & \multicolumn{1}{c}{Architecture} \\
\midrule
Encoder & fc(unit=784)+ReLU \\
&+ fc(unit=512)+ReLU\\
&+ fc(unit=256)\\
&+ fc(unit=2)\\
[5pt]
Decoder & fc(unit=2)+ReLU \\
&+ fc(unit=256)+ReLU\\
&+ fc(unit=512)+ReLU\\
&+ fc(unit=784)+Sigmoid\\
\bottomrule
\end{tabular}}
\end{small}
\end{sc}
\end{minipage}
\end{table}

\begin{table}
\begin{minipage}{\linewidth}
\centering
\caption{Cifar10 classification experiment network architecture. \label{tab:cifar_net} }
\setlength{\tabcolsep}{25pt}
\begin{sc}
\begin{small}
\resizebox{.9\textwidth}{!}{
\begin{tabular}{cl}
\toprule
\multicolumn{1}{c}{Network} & \multicolumn{1}{c}{Architecture} \\
\midrule
Network & conv2d(unit=6, kernel=5, stride=1)+ReLU\\
&+maxpool2d(stride=2)\\
&+ conv2d(unit=16, kernel=5, stride=1)+ReLU\\
&+maxpool2d(stride=2)\\
&+ fc(unit=120)+ReLU\\
&+ FC(unit=84)+ReLU\\
&+ FC(unit=10)\\
\bottomrule
\end{tabular}}
\end{small}
\end{sc}
\end{minipage} 
\end{table}

\begin{table}
\begin{minipage}{\linewidth}
\centering
\caption{Cifar10 VAE experiment network architecture. \label{tab:cifar_net} }
\setlength{\tabcolsep}{25pt}
\begin{sc}
\begin{small}
\resizebox{.9\textwidth}{!}{
\begin{tabular}{cl}
\toprule
\multicolumn{1}{c}{Network} & \multicolumn{1}{c}{Architecture} \\
\midrule
Encoder & conv2d(unit=16, kernel=3, stride=1) +BN+ReLU\\
&+ conv2d(unit=32, kernel=3, stride=2)+BN+ReLU\\
&+ conv2d(unit=32, kernel=3, stride=1)+BN+ReLU\\
&+ conv2d(unit=16, kernel=3, stride=2)+BN+ReLU\\
&+ FC(unit=512)+BN+ReLU\\
[5pt]
Decoder & FC(unit=512)+BN+ReLU\\
& FC(unit=1024)+BN+ReLU\\
& deconv(unit=32, kernel=3, stride=2)+BN+ReLU\\
& deconv(unit=32, kernel=3, stride=1)+BN+ReLU\\
& deconv(unit=16, kernel=3, stride=2)+BN+ReLU\\
& deconv(unit=768, kernel=3, stride=2)\\
\bottomrule
\end{tabular}}
\end{small}
\end{sc}
\end{minipage} 
\end{table}

\begin{table}
\begin{minipage}{0.5\linewidth}
\centering
\caption{Tiny Imagenet classification experiment network architecture. \label{tab:cifar_net} }
\setlength{\tabcolsep}{25pt}
\begin{sc}
\begin{small}
\resizebox{.9\textwidth}{!}{
\begin{tabular}{cl}
\toprule
\multicolumn{1}{c}{Network} & \multicolumn{1}{c}{Architecture} \\
\midrule
Network & resnet18\footnote{Resnet 18 without last layer}\\
&+ FC(unit=200)\\
\bottomrule
\end{tabular}}
\end{small}
\end{sc}
\end{minipage} 
\end{table}
\end{document}

% \twocolumn[
%\icmltitle{Accelerating Decentralized Stochastic Optimization with Finite-time Convergence Gossip Algorithm}
% \icmltitle{Supplementary Material for "Finite-Time Consensus Learning for\\Decentralized Optimization with Nonlinear Gossiping"}

\icmlsetsymbol{equal}{*}

\begin{icmlauthorlist}
\icmlauthor{Junya Chen}{duke,fudan}
\icmlauthor{Chenyang Tao}{duke}
\icmlauthor{Wenlian Lu}{fudan}
\icmlauthor{Lawrence Carin}{duke}
\end{icmlauthorlist}

%\icmlaffiliation{to}{Department of Computation, University of Torontoland, Torontoland, Canada}
%\icmlaffiliation{goo}{Googol ShallowMind, New London, Michigan, USA}
%\icmlaffiliation{ed}{School of Computation, University of Edenborrow, Edenborrow, United Kingdom}

\icmlaffiliation{duke}{Electrical \& Computer Engineering, Duke University, Durham, NC, USA}
\icmlaffiliation{fudan}{School of Mathematical Sciences, Fudan University, Shanghai, China}
%\icmlaffiliation{tsinghua}{Computer Science \& Technology, Tsinghua University, Beijing, China}
%\icmlaffiliation{rti}{RTI International, Research Triangle Park, NC, USA}

\icmlcorrespondingauthor{Cieua Vvvvv}{c.vvvvv@googol.com}
\icmlcorrespondingauthor{Eee Pppp}{ep@eden.co.uk}

\icmlkeywords{Machine Learning, ICML}

\vskip 0.3in
]



\printAffiliationsAndNotice{\icmlEqualContribution} % otherwise use the standard text.

\subsection{Experiments}
All experiments are implemented with Tensorflow and run on a single NVIDIA TITAN X GPU. 
\begin{figure}[t]
	\begin{center}
	    \includegraphics[width = 0.23\textwidth]{icml-2020/1000.pdf}
		\includegraphics[width = 0.23\textwidth]{icml-2020/4.pdf}
		\hspace{-1.2em}
		\includegraphics[width = 0.23\textwidth]{icml-2020/1001.pdf}
		\includegraphics[width = 0.23\textwidth]{icml-2020/0033.pdf}
		\hspace{-1.2em}
		\includegraphics[width = 0.23\textwidth]{icml-2020/1003.pdf}
		\includegraphics[width = 0.23\textwidth]{icml-2020/12.pdf}
		\hspace{-1.2em}
		\includegraphics[width = 0.23\textwidth]{icml-2020/1004.pdf}
		\includegraphics[width = 0.23\textwidth]{icml-2020/1063.pdf}
		\hspace{-1.2em}
		\includegraphics[width = 0.23\textwidth]{icml-2020/1005.pdf}
		\includegraphics[width = 0.23\textwidth]{icml-2020/1292.pdf}
		\caption{Convergence of NGO SGD under non-\iid setting in the VAE task. (a) Generations on each worker after 2 iterations. (b) Generations on each worker after 200 iterations.\label{noniid_1}}
		\vspace{-1em}
	\end{center}
\end{figure}

\begin{table}[t!]\label{table1}
\caption{Comparison of NGO and Gossip for MNIST classification \label{tab:2}}
\setlength{\tabcolsep}{3pt}
\begin{center}
\begin{small}
\begin{sc}
\begin{tabular}{lcccccccccccc}
\toprule
Setting& \multicolumn{3}{c}{Gossip} & \multicolumn{2}{c}{NGO} \\
\midrule
\iid Centralized&\multicolumn{5}{c}{0.26}\\
\iid Ring&\multicolumn{3}{c}{0.25(50}&\multicolumn{2}{c}{0.25(48)}\\
\iid Random&\multicolumn{3}{c}{0.25(44)}&\multicolumn{2}{c}{0.25(44)} \\
\iid FC &\multicolumn{4}{c}{0.25(32)}&\multicolumn{2}{c}{0.25(28)}\\
\midrule
non\ \iid Centralized&\multicolumn{5}{c}{0.72}\\
non\ \iid Ring&\multicolumn{3}{c}{0.72(10)}&\multicolumn{2}{c}{0.71(8)}\\
non\ \iid Random&\multicolumn{3}{c}{0.73(5)}&\multicolumn{2}{c}{0.72(4)}\\
non\ \iid FC &\multicolumn{4}{c}{0.72(4)}&\multicolumn{2}{c}{0.72(4)}\\
\bottomrule
\end{tabular}
\end{sc}
\end{small}
\end{center}
\end{table}

\begin{table}[t!]\label{table1}
\caption{Comparison of NGO and Gossip for CIFAR10 classification\label{tab:3}}
\setlength{\tabcolsep}{3pt}
\begin{center}
\begin{small}
\begin{sc}
\begin{tabular}{lcccccccccccc}
\toprule
Setting& \multicolumn{3}{c}{Gossip} & \multicolumn{2}{c}{NGO} \\
\midrule
\iid Centralized&\multicolumn{5}{c}{0.16}\\
\iid Ring&\multicolumn{3}{c}{0.54}&\multicolumn{2}{c}{0.22}\\
\iid Random&\multicolumn{3}{c}{0.19}&\multicolumn{2}{c}{0.14} \\
\iid FC &\multicolumn{4}{c}{~~~0.30}&\multicolumn{2}{c}{~~~0.38}\\
\midrule
non\ \iid Centralized&\multicolumn{5}{c}{0.18}\\
non\ \iid Ring&\multicolumn{3}{c}{0.50}&\multicolumn{2}{c}{0.50}\\
non\ \iid Random&\multicolumn{3}{c}{0.25}&\multicolumn{2}{c}{0.24}\\
non\ \iid FC &\multicolumn{4}{c}{~~~0.34}&\multicolumn{2}{c}{~~~0.32}\\
\bottomrule
\end{tabular}
\end{sc}
\end{small}
\end{center}
\end{table}

\subsubsection{Supervised Learning}
We apply NGO and Gossip SGD on MNIST and Cifar 10 classification tasks. Quantitative results are summarized in Table \ref{tab:2} and \ref{tab:3}. For both tasks, we consider different network topologies, including ring, random network and fully connected (FC) under \iid setting and non \iid setting. We also provide the performance of centralized learning in the tables. For MNIST task, we set the workers as fully connected neural networks. Since all these protocols can reach the same accuracy in the end, we provide the first hitting time of the final accuracy in the brackets. In the Cifar10 task, we set the workers as convolutional neural networks. NGO SGD outperforms Gossip SGD in both MNIST and CIFAR 10 classification tasks.

\subsubsection{Unsupervised Learning}
We now compare the performance of NGO and Gossip for unsupervised learning problem. Here we adopt the VAE setting, where one seeks to maximize the lower bound to the data likelihood b optimizing the stochastic autoencoder $q(z)$ and conditional data likelihood $p(x|z)$. We compare the result of Gossip SGD and NGO SGD, together with the centralized learning result under \iid setting and non-\iid setting. In Table \ref{table1}, we compare the evidence lower bound and the results show that NGO SGD outperforms SGD under both settings, and NGO performs almost as good as centralized learning. We provide a few generations on each worker compared to their initial non \iid distribution in Figure \ref{noniid_1}.

\subsubsection{Scalability}
Figure \ref{scalability} plots the scalability with Gossip SGD and NGO SGD by scaling up the worker numbers on ring topology under the \iid setting. Gossip SGD achieve almost linear speedup on throughput at the beginning of training, while NGO SGD benefits from more workers during the whole 
training process.

\begin{figure}[t]
	\begin{center}
	\hspace{-1.5em}
	    \includegraphics[width=0.5\textwidth]{icml-2020/Workers.pdf}
		\caption{Speedup achieved by scaling up the worker numbers. (a) Gossip SGD. (b) NGO SGD}\label{scalability}
	\end{center}
\end{figure}

\subsection{Proof of Theorem 3.1}
\begin{proof}
Recall the definition of Lyapunov function
\begin{equation*}
    V(t) = \sum_{i=1}^n\left(\xv_i(t) - \bar{\xv}\right)^2,
\end{equation*}
after differentiating we have:
\begin{equation}
	\begin{split}
	\dot{V}(t) &= 2\sum_{i}\left(\xv_i(t) - \bar{\xv}\right)(\dot{\xv}_i(t) - \dot{\bar{\xv}})\\
	&= 2\gamma\sum_i \left(\xv_i(t) - \bar{\xv}(t)\right)\sum_{j:(i,j)\in \mathcal{E}} \Wmat_{ij}\phiv\left(\xv_j(t) - \xv_i(t)\right)
	\end{split}
\end{equation}

Here we use the fact that $\sum_{(i,j)\in \mathcal{E}}\Wmat_{ij}\phiv(\xv_j(t)-\xv_i(t)) = 0, \forall t>0$, thus $\bar{\xv}(t) = \bar{\xv}$ is invariant.

\begin{equation}
	\begin{split}
	&\sum_{i}\sum_{j:(i,j)\in \mathcal{E}} \Wmat_{ij}\left|\xv_i(t) -\xv_j(t)\right|\phiv\left(\left|\xv_j(t) - \xv_i(t)\right|\right)\\
	=& \sum_{(i,j)\in \mathcal{E}}\Wmat_{ij}\left(\xv_i(t) - \xv_j(t)\right)^{2p}\\
	\geq & \left[\sum_{(i,j)\in\mathcal{E}} \Wmat_{ij}^{\frac{1}{p}}\left(\xv_i(t) - \xv_j(t)\right)^2 \right]^{p}\\
	=& \left[\sum_{(i,j)\in\mathcal{E}} \Wmat_{ij}^{\frac{1}{p}}\left(\deltav_i(t) - \deltav_j(t)\right)^2 \right]^{p}\\
	%&\geq \underline{a} \sum_{ij} \mu\left(\sum_{ij}|x_i(t) - x_j(t)|^2\right)
	\end{split}
\end{equation}
where $\Bmat =\left[\Wmat_{ij}^{\frac{1}{p}}\right]\in \BR^{n\times n}$. Notice that $\sum_{i,j=1}^{n} \Wmat_{ij}^\frac{1}{p}(\deltav_i-\deltav_j)^2 = 2\deltav^\top \Lmat(\Bmat)\deltav$, and $\sum_i \deltav_i = 0$,
we have
\begin{equation*}
	\frac{\sum_{i,j = 1}^n \Wmat_{ij}^{\frac{1}{p}} \left(\deltav_i-\deltav_j\right)^2}{V(t)} = \frac{2\deltav^\top \Lmat(\Bmat)\deltav}{\deltav^\top\deltav} \geq 2\lambda_2(\Lmat(\Bmat)).
\end{equation*}
where $\Lmat(\Bmat)$ is the graph Laplacian of $\mathcal{G}(\Bmat)$ and $\lambda_2(\Lmat(\Bmat))$ is the algebraic connectivity of $\mathcal{G}(\Bmat)$, $\lambda_2(\Lmat(\Bmat))$ is positive since the graph is connected.

Therefore,
\begin{equation}
\begin{split}
	\dot{V}(t) &= 2\gamma\sum_i(\xv_i(t)-\bar{\xv})\sum_{j:(i,j)\in \mathcal{E}} \Wmat_{ij}\phiv(\xv_j(t)-\xv_i(t))\\
	&=2\gamma\sum_{(i,j)\in \mathcal{E}} \Wmat_{ij}\xv_i(t)\phiv(\xv_j(t) - \xv_i(t))\\
	&=2\gamma\sum_{(i,j)\in \mathcal{E}} \Wmat_{ij}(\xv_i(t) - \xv_j(t))\phiv(\xv_j(t) - \xv_i(t))\\
	&=-2\gamma\sum_{(i,j)\in \mathcal{E}} \Wmat_{ij}|\xv_i(t) -\xv_j(t)|\phiv \left(|\xv_j(t) - \xv_i(t)|\right)\\
	&\leq -4\gamma \lambda_2^p(\Lmat(\Bmat)) V(t)^{p}
	\end{split}
\end{equation}

Therefore $V(t) = 0$, when $t>t^* = \frac{V^{1-p}(0)}{4\gamma \lambda_2^p(\Lmat(\Bmat))(1-p)}$.
\end{proof}
\subsection{Convergence rate}
\begin{proof}

Step B for Algorithm 2 and Algorithm 3:
\begin{lem}
	The average $\bar{\xv}^{(t)}$ of the iterates of  Algorithm 2 and Algorithm 3 satisfy the following:
	\begin{equation*}
	\begin{split}
	&\mathbb{E}_{\xi_1^{(t)}, \cdots, \xi_n^t}\left\|\bar{\xv}^{(t+1)} -\xv^*\right\|^2\\
	\leq& \left(1-\mu\eta_t\right)\left\|\bar{\xv}^{(t)}-\xv^*\right\|^2 +\frac{\eta_t^2\bar{\sigma}^2}{n} 
	-2\eta_t\left[1-\left(1+\frac{1}{\nu_t}\right)L\eta_t\right]\left(f(\bar{\xv}^t) - f^*\right) \\
	+&\eta_t\frac{\left(1+\nu_t \right)\eta_tL^2+L}{n}\sum_{i=1}^{n}\|\bar{\xv}^t - \xv_i^t\|^2
	\end{split}
	\end{equation*}
	where $\bar{\sigma}^2 = \frac{1}{n}\sum_{i=1}^{n}\sigma_i^2$, $\nu$ is any positive constant, and set 
	\begin{equation*}
	\left(\nu_t^*\right)^2 = \frac{\left\|\sum_{i=1}^n \nabla f_i(\bar{\xv}^t)\right\|^2}{\left\|\sum_{i=1}^n \nabla f_i(\xv_i^t) - \sum_{i=1}^{n}\nabla f_i(\bar{\xv}^t)\right\|^2},
	\end{equation*}
	and any $\nu \neq \nu^* $ will lead to a looser bound.
\end{lem}
\begin{proof}
	\begin{equation*}
	\begin{split}
	&\left\|\bar{\xv}^{(t+1)} - \xv^*\right\|^2\\ 
	=& \left\|\bar{\xv}^{(t)} - \frac{\eta_t}{n}\sum_{j=1}^n \nabla F_j\left(\xv_j^{(t)}, \xi_j^{(t)}\right)- \xv^*\right\|^2\\
	=&\Bigg\|\bar{\xv}^{(t)} - \xv^*-\frac{\eta_t}{n}\sum_{i=1}^n\nabla f_i\left(\xv_i^{(t)}\right) +\frac{\eta_t}{n}\sum_{i=1}^n\nabla f_i\left(\xv_i^{(t)}\right)-\frac{\eta_t}{n}\sum_{j=1}^n \nabla F_j\left(\xv_j^{(t)}, \xi_j^{(t)}\right)\Bigg\|^2\\
	=&\left\|\bar{\xv}^{(t)} - \xv^*-\frac{\eta_t}{n}\sum_{i=1}^n\nabla f_i\left(\xv_i^{(t)}\right)\right\|^2+ \eta_t^2\left\|\frac{1}{n}\sum_{i=1}^n\nabla f_i\left(\xv_i^{(t)}\right)- \sum_{j=1}^n \nabla F_j\left(\xv_j^{(t)}, \xi_j^{(t)}\right)\right\|^2 \\
	&+2\frac{\eta_t}{n}\bigg\langle \bar{\xv}^{(t)} -\xv^*-\frac{\eta_t}{n}\sum_{i=1}^n \nabla f_i\left(\xv_i^{(t)}\right),\sum_{i=1}^n \nabla f_i\left(\xv_i^{(t)}\right)- \sum_{j=1}^n\nabla F_j\left(\xv_j^{(t)}, \xi_j^{(t)}\right) \bigg\rangle.\\
	\end{split}
	\end{equation*}
	The last term is zero in expectation.  
	The second term is less than $\frac{\eta_t^2\bar{\sigma}^2}{n}$ (bounded variance assumption).
	
	The first term can be written as:
	\begin{equation}\label{first term}
	\begin{split}
	&\left\|\bar{\xv}^{(t)} - \xv^*-\frac{\eta_t}{n}\sum_{i=1}^n\nabla f_i\left(\xv_i^{(t)}\right)\right\|^2 \\ 
	=& \left\|\bar{\xv}^{(t)} - \xv^*\right\|^2 +\eta_t^2\left\|\frac{1}{n}\sum_{i=1}^n \nabla f_i\left(\xv_i^{(t)}\right)\right\|^2 -2\eta_t\left\langle \bar{\xv}^{(t)} - \xv^*, \frac{1}{n}\sum_{i=1}^n \nabla f_i\left(\xv_i^{(t)}\right)\right\rangle\\
	\end{split}
	\end{equation}
	
	The second term in (\ref{first term}):
	\begin{equation*}
	\begin{split}
	    &\left\|\frac{1}{n}\sum_{i=1}^n \nabla f_i\left(\xv_i^{(t)}\right)\right\|^2\\
	    =&\Big\|\frac{1}{n}\sum_{i=1}^n \left(\nabla f_i\left(\xv_i^{(t)}\right) - \nabla f_i\left(\bar{\xv}^{(t)}\right)\right)+\left(\nabla f_i\left(\bar{\xv}^{(t)}\right)-\nabla f_i(\xv^*)\right)\Big\|^2\\
		\leq & \frac{(1+\nu)}{n}\sum_{i=1}^{n}\left\|\nabla f_i
		\left(\xv_i^{(t)}\right) - \nabla f_i\left(\bar{\xv}^{(t)}\right)\right\|^2 + \left(1+\frac{1}{\nu}\right)\left\|\frac{1}{n}\sum_{i=1}^n \nabla f_i\left(\bar{\xv}^{(t)}\right) -\frac{1}{n}\sum_{i=1}^n \nabla f_i(\xv^*)\right\|^2\\
		\leq & \frac{(1+\nu)L^2}{n}\sum_{i=1}^{n}\left\|\xv_i^{(t)} - \bar{\xv}^{(t)}\right\|^2 + \left(1+\frac{1}{\nu}\right)\frac{2L}{n}\sum_{i=1}^{n}\left(f_i\left(\bar{\xv}^{(t)}\right) -f_i(\xv^*)\right)\\
		=& \frac{(1+\nu)L^2}{n}\sum_{i=1}^{n}\left\|\xv_i^{(t)} - \bar{\xv}^{(t)}\right\|^2 + 2\left(1+\frac{1}{\nu}\right)L\left(f\left(\bar{\xv}^{(t)}\right) - f^*\right),
	\end{split}	
	\end{equation*}
	 $\forall \nu>0$
	where the optimal value of $\nu$ is 
	\begin{equation*}
		\left(\nu^*\right)^2 = \frac{\left\|\sum_{i=1}^n \nabla f_i(\bar{\xv}^t)\right\|^2}{\left\|\sum_{i=1}^n \nabla f_i(\xv_i^t) - \sum_{i=1}^{n}\nabla f_i(\bar{\xv}^t)\right\|^2}
	\end{equation*}
		
	The third term in (\ref{first term}):
	\begin{equation*}
	\begin{split}
		&-2\eta_t\left\langle \bar{\xv}^{(t)} - \xv^*, \frac{1}{n}\sum_{i=1}^n \nabla f_i\left(\xv_i^{(t)}\right)\right\rangle\\
		=&-\frac{2\eta_t}{n}\sum_{i=1}^n\Big[\left\langle \bar{\xv}^{(t)}-\xv_i^{(t)}, \nabla f_i\left(\xv_i^{(t)}\right)\right\rangle+\left\langle \xv_i^{(t)}-\xv^*, \nabla f_i\left(\xv_i^{(t)}\right) \right\rangle\Big]\\
		\leq& -\frac{2\eta_t}{n}\sum_{i=1}^n\bigg[f_i\left(\bar{\xv}^{(t)}\right) - f_i\left(\xv_i^{(t)}\right)-\frac{L}{2}\left\|\bar{\xv}^{(t)} - \xv_i^{(t)}\right\|^2+f_i\left(\xv_i^{(t)}\right) - f_i\left(\xv^*\right)+\frac{\mu}{2}\left\|\xv_i^{(t)}- \xv^*\right\|^2\bigg]\\
		\leq &-2\eta_t \left(f\left(\bar{\xv}^{(t)}\right)-f(\xv^*)\right)+\frac{L\eta_t}{n}\sum_{i=1}^{n}\left\|\bar{\xv}^{(t)}- \xv_i^{(t)}\right\|^2-\mu \eta_t\left\|\bar{\xv}^{(t)} - \xv^*\right\|^2
	\end{split}
	\end{equation*}
	Combining everything together, we have the statement of the lemma.
\end{proof}

Step C for Algorithm 2:
\begin{proof}
	According to Lemma 21 in \cite{koloskova2019decentralized}, we have:
	\begin{equation*}
		\|{\xv}^{(t+1)} - \bar{\xv}^{(t+1)}\|_F^2 \leq 40\eta_t^2\frac{1}{\beta^2}nG^2
	\end{equation*}
	where $\beta =1-\gamma\lambda_2(\Lmat(\Wmat)) $.
	
	We take 
	\begin{equation*}
	    \nu_t^2= \min\left\{\frac{\left\|\sum_{i=1}^n \nabla f_i(\bar{\xv}^t) 
	    \right\|^2}{\left\|\sum_{i=1}^n \nabla f_i(\xv_i^t) - \sum_{i=1}^{n}\nabla f_i(\bar{\xv}^t)\right\|^2} , \frac{1}{\eta_t^{2\varepsilon}} \right\},
	\end{equation*}
	where $0<\epsilon<1$.
	
	For $\eta_t\leq \frac{1}{4L}$ it holds $(1+\frac{1}{\nu_t})L\eta_t - 1\leq \frac{1}{4\nu_t} - \frac{3}{4}$, $(1+\nu_t) \eta_tL^2 +L \leq \frac{5+\nu_t}{4}L$,
	hence
	\begin{equation*}
	\begin{split}
	&\mathbb{E}\left\|\bar{\xv}^{(t+1)} -\xv^*\right\|^2\\ 
	\leq& (1-\mu\eta_t)\mathbb{E}\left\|\bar{\xv}^{(t)}- \xv^*\right\|^2 +2\eta_t \left(\frac{1}{4\nu_t}-\frac{3}{4}\right)e_t +\frac{\eta_t^2\bar{\sigma}^2}{n}
	+\eta_t^3 \frac{\left(50+10\nu_t\right)L}{\beta^2}nG^2\\
	\leq& (1-\mu\eta_t)\mathbb{E}\|\bar{\xv}^{(t)}- \xv^*\|^2 +2\eta_t \left(\frac{1}{4\nu_t}-\frac{3}{4}\right)e_t +\frac{\eta_t^2\bar{\sigma}^2}{n}
	+\frac{10\eta_t^{3-\epsilon}L}{\beta^2}nG^2 + \frac{50}{\beta^2} \eta_t^3LnG^2
	\end{split}
	\end{equation*}
	From Lemma 23 in \cite{koloskova2019decentralized} we get
	\begin{equation*}
	\begin{split}
	&\left(\frac{3}{2}- \frac{1}{2\nu_{\min}}\right)\sum_{i=0}^{T-1}\frac{w_te_t}{S_T}\leq \frac{\mu a_0a^3}{4S_T}+\frac{2\bar{\sigma}^2T(T+2a)}{\mu n S_T}
	+C_1\frac{(a+T)^{1+\epsilon}}{S_T\beta^2}nG^2L +C_2\frac{1}{S_T\beta^2}nG^2L
	\end{split}
	\end{equation*}
\end{proof}

Step C for Algorithm 3:
\begin{proof}
	According to the analysis in step A, we have:
	\begin{equation*}
		\left\|{\xv}^{(t+1)} - \bar{\xv}^{(t+1)}\right\|_F^2 = 0
	\end{equation*}
	after $t>T_0$.
	
	For $\eta_t\leq \frac{1}{4L}$ it holds $(1+\frac{1}{\nu_t})L\eta_t - 1\leq \frac{1}{4\nu_t} - \frac{3}{4}$, $(1+\nu_t) \eta_tL^2 +L \leq \frac{5+\nu_t}{4}L$,
	hence
	\begin{equation*}
	\begin{split}
	\mathbb{E}\left\|\bar{\xv}^{(t+1)} -\xv^*\right\|^2
	\leq (1-\mu\eta_t)\mathbb{E}\|\bar{\xv}^{(t)}- \xv^*\|^2 +2\eta_t \left(\frac{1}{4\nu_t}-\frac{3}{4}\right)e_t +\frac{\eta_t^2\bar{\sigma}^2}{n}\\
	\end{split}
	\end{equation*}
	From Lemma 23 in \cite{koloskova2019decentralized} we get
	\begin{equation*}
	\left(\frac{3}{2}- \frac{1}{2\nu_{\min}}\right)\sum_{i=T_0}^{T-1}\frac{w_te_t}{S_T}\leq \frac{\mu a_0a^3}{4S_T}+\frac{2\bar{\sigma}^2T(T+2a)}{\mu n S_T}
	\end{equation*}
\end{proof}

Step B for Algorithm 1:

Denote $\hat{F}(\xv^t, \xi_1^t, \xi_2^t, \cdots, \xi_n^t)  = \frac{1}{n}\sum_{i=1}^nF_i(\xv^t, \xi_i^t)$, then
$\mathbb{E}_{\xi_1^t, \cdots, \xi_n^t}\hat{F}(\xv^t, \xi_1^t, \xi_2^t, \cdots, \xi_n^t) =  f(\xv^t)$.
\begin{lem}
	\begin{equation*}
	\begin{split}
		\mathbb{E}_{\xi_1^{(t)}, \cdots, \xi_n^{(t)}}\left\|\xv^{(t+1)} - \xv^*\right\|^2\leq \left(1-\mu\eta_t\right)\left\|\xv^{(t)}-\xv^*\right\|^2 +\frac{\eta_t^2\bar{\sigma}^2}{n} -2\eta_t\left(1-L\eta_t\right)\left(f\left(\xv^{(t)}\right) - f^*\right)
	\end{split}
	\end{equation*}
\end{lem}
\begin{proof}
\begin{equation*}
\begin{split}
&\left\|\xv^{(t+1)} - \xv^*\right\|^2 \\
=&\left\|\xv^{(t)} - \eta_t \nabla \hat{F}\left(\xv^{(t)}, \xi_1^{(t)}, \cdots, \xi_n^{(t)}\right)- \xv^*\right\|^2\\
=&\Bigg\|\xv^{(t)} - \xv^*-\eta_t\nabla f\left(\xv^{(t)}\right) +\eta_t\nabla f\left(\xv^{(t)}\right)-\eta_t \nabla\hat{F}\left(\xv^{(t)}, \xi_1^t, \cdots, \xi_n^{(t)}\right)\Bigg\|^2\\
=& \left\|\xv^{(t)} -\xv^*-\eta_t\nabla f\left(\xv^{(t)}\right)\right\|^2 + \eta_t^2\Bigg\|\nabla f\left(\xv^{(t)}\right)-  \nabla \hat{F}\left(\xv^{(t)}, \xi_1^{(t)}, \cdots, \xi_n^{(t)}\right)\Bigg\|^2 \\
-&2\eta_t\Big\langle \xv^{(t)} -\xv^*-\eta_t\nabla f\left(\xv^{(t)}\right),\nabla f\left(\xv^{(t)}\right)- \nabla \hat{F}(\xv^{(t)}, \xi_1^{(t)}, \cdots, \xi_n^{(t)}) \Big\rangle\\
\end{split}
\end{equation*}
The last term is zero in expectation.  
The second term can be written as:
\begin{equation*}
\begin{split}
&\mathbb{E}_{\xi_1^t, \cdots, \xi_n^t}\left\|\nabla F(\xv^t)-  \nabla \hat{F}(\xv^t, \xi_1^t, \cdots, \xi_n^t)\right\|^2\\
=&\mathbb{E}_{\xi_1^t, \cdots, \xi_n^t}\left\|\frac{1}{n}\sum_{i=1}^n \nabla  f_i(\xv^t) - \frac{1}{n}\sum_{i=1}^{n}\nabla F_i(\xv^t, \xi_i^t)\right\|^2\\
=&\frac{1}{n^2}\sum_{i=1}^{n}\mathbb{E}_{\xi_1^t, \cdots, \xi_n^t}\left\|\nabla f_i(\xv^t) - \nabla F_i(\xv^t,\xi_i)\right\|^2 \\
=& \frac{1}{n^2}\sum_{i=1}^{n}\sigma_i^2 = \frac{\bar{\sigma}^2}{n}
\end{split}
\end{equation*}
The first term can be written as:
\begin{equation}
\begin{split}
	 &\left\|\xv^{(t)} - \xv^*-\eta_t\nabla f(\xv^{(t)})\right\|^2 \\ 
	 =& \left\|\xv^{(t)} - \xv^*\right\|^2 +\eta_t^2\left\|\nabla f(\xv^t)\right\|^2
	 -2\eta_t\left\langle \xv^{(t)} - \xv^*, \nabla f(\xv^{(t)})\right\rangle\\
	 =&  \left\|\xv^{(t)} -\xv^*\right\|^2 +\eta_t^2\left\|\nabla f(\xv^{(t)})- \nabla f(\xv^*)\right\|^2- 2\eta_t\left\langle \xv^{(t)}-\xv^*, \nabla f(\xv^{(t)})\right\rangle\\
	 \leq&\|\xv^{(t)}-  \xv^*\|^2 +2L\eta_t^2\left(f(\xv^{(t)}) - f(\xv^*)\right)-2\eta_t\left(f(\xv^*) - f(\xv^t)-\frac{\mu}{2}\|\xv^{(t)}- \xv^*\|^2\right)
\end{split}
\end{equation}
Putting everything together we are getting the statement of the lemma.
\end{proof}

Step C for Algorithm 1:
\begin{proof}
	For $\eta_t\leq \frac{1}{4L}$ it holds $L\eta_t - 1\leq -\frac{3}{4}$, hence
	\begin{equation}
		\mathbb{E}\|\xv^{(t+1)} - \xv^*\|^2 \leq (1-\mu\eta_t)\mathbb{E}\|\xv^{(t)}- \xv^*\|^2 +\frac{\eta_t^2\bar{\sigma}^2}{n} -\frac{3}{2}\eta_te_t
	\end{equation}
	From Lemma 23 in \cite{koloskova2019decentralized} we get
	\begin{equation*}
	\frac{1}{S_T}\sum_{i=0}^{T-1}w_te_t\leq \frac{\mu a^3}{6S_T}a_0 +\frac{4T(T+2a)}{3\mu S_T}\frac{\bar{\sigma}^2}{n}
	\end{equation*}
\end{proof}

\bibliography{ngo}
\bibliographystyle{icml2020}

\clearpage